\documentclass[10pt]{iopart}
\usepackage{amsmath,amsfonts,bm}

\def\twofigref#1#2{Figures \ref{#1} and \ref{#2}}

\def\Secref#1{Section~\ref{#1}}
\def\eqref#1{equation~\ref{#1}}
\def\plaineqref#1{(\ref{#1})}
\def\1{\bm{1}}

\def\vt{{\bm{t}}}
\def\vu{{\bm{u}}}

\def\vx{{\bm{x}}}

\DeclareMathAlphabet{\mathsfit}{\encodingdefault}{\sfdefault}{m}{sl}
\SetMathAlphabet{\mathsfit}{bold}{\encodingdefault}{\sfdefault}{bx}{n}

\def\gE{{\mathcal{E}}}

\def\gL{{\mathcal{L}}}

\def\gN{{\mathcal{N}}}
\def\gO{{\mathcal{O}}}

\def\gQ{{\mathcal{Q}}}
\def\gR{{\mathcal{R}}}

\def\gT{{\mathcal{T}}}

\newcommand{\R}{\mathbb{R}}

\usepackage{times}
\usepackage{fullpage}
\usepackage{amsmath,amsthm,amssymb,amsxtra,mathtools,bm,cancel}
\usepackage{array}
\usepackage{graphicx}
\usepackage{algorithm}
\usepackage{algpseudocode}
\usepackage{physics}
\usepackage{MnSymbol,wasysym}
\usepackage{tikzsymbols}
\usepackage{array,tabularx,multirow}
\usepackage{makecell}
\usepackage{array}
\usepackage{caption}
\usepackage{subcaption}
\usepackage{paralist}
\newcolumntype{P}[1]{>{\centering\arraybackslash}p{#1}}

\usepackage{verbatim}
\usepackage{enumerate}
\usepackage{parskip}

\usepackage{color,soul}
\definecolor{clemson-orange}{RGB}{234,106,32}
\definecolor{highlight-orange}{RGB}{255,150,150}
\definecolor{chicago-maroon}{RGB}{128,0,0}
\definecolor{cincinnati-red}{RGB}{190,0,0}
\definecolor{soft-cyan}{RGB}{68,85,90}
\definecolor{firebrick}{RGB}{178,34,34}
\definecolor{crimson}{RGB}{220,20,60}
\definecolor{cerrulean}{rgb}{0.165,0.322,0.745}
\definecolor{jaam}{rgb}{0.45,0.0,0.45}

\usepackage{hyperref}
\hypersetup{
  colorlinks   = true,    % Colours links instead of ugly boxes
  urlcolor     = jaam,    % Colour for external hyperlinks
  linkcolor    = firebrick,    % Colour of internal links
  citecolor    = blue      % Colour of citations
}
\usepackage{parskip}

\usepackage{thmtools}
\declaretheoremstyle[
    headfont=\bfseries, 
    bodyfont=\normalfont\itshape, spaceabove=10pt,
    spacebelow=10pt]{mystyle}
\theoremstyle{mystyle}
\newtheorem{theorem}{Theorem}[section]
\newtheorem{lemma}[theorem]{Lemma}

\newif\ifsolutions \solutionstrue

\def\final{0}
\newcommand{\reviewer}[3]{
  \expandafter\newcommand\csname #1\endcsname[1]{
    \ifthenelse{\equal{\final}{1}} {
      \textcolor{#3}{}
    } {
      \textcolor{#3}{\begin{center} \textbf{#2} ##1 \end{center}}
    }
  }
}

\reviewer{anirbit}{}{jaam}

\newcommand{\bb}{\mathbb}

\newcommand{\N}{{\bb N}}

\renewcommand{\P}{\mathcal{P}}

\def\1{\bm{1}}

\def\vt{{\bm{t}}}
\def\vu{{\bm{u}}}

\def\vx{{\bm{x}}}

\DeclareMathAlphabet{\mathsfit}{\encodingdefault}{\sfdefault}{m}{sl}
\SetMathAlphabet{\mathsfit}{bold}{\encodingdefault}{\sfdefault}{bx}{n}
\def\gE{{\mathcal{E}}}

\def\gL{{\mathcal{L}}}

\def\gN{{\mathcal{N}}}
\def\gO{{\mathcal{O}}}

\def\gQ{{\mathcal{Q}}}
\def\gR{{\mathcal{R}}}

\def\gT{{\mathcal{T}}}

\allowdisplaybreaks[1]

\begin{document}
\title{Investigating the Ability of PINNs To Solve Burgers' PDE Near Finite-Time BlowUp}
% \title[]{Investigating the Ability of PINNs To Solve Burgers' PDE Near Finite-Time BlowUp}

\author{Dibyakanti Kumar}
\address{Barclays, India}
\ead{dibyakanti.kumar@barclays.com}
\author{Anirbit Mukherjee}
\address{Department of Computer Science, University of Manchester, Manchester, UK}
\ead{anirbit.mukherjee@manchester.ac.uk}
% \vspace{10pt}
% \begin{indented}
% \item[]August 2017
% \end{indented}

\begin{abstract}
Physics Informed Neural Networks (PINNs) have been achieving ever newer feats of solving complicated PDEs numerically while offering an attractive trade-off between accuracy and speed of inference. A particularly challenging aspect of PDEs is that there exist simple PDEs which can evolve into singular solutions in finite time starting from smooth initial conditions. In recent times some striking experiments have suggested that PINNs might be good at even detecting such finite-time blow-ups. In this work, we embark on a program to investigate this stability of PINNs from a rigorous theoretical viewpoint. Firstly, we derive generalization bounds for PINNs for Burgers' PDE, in arbitrary dimensions, under conditions that allow for a finite-time blow-up. Then we demonstrate via experiments that our bounds are significantly correlated to the $\ell_2$-distance of the neurally found surrogate from the true blow-up solution, when computed on sequences of PDEs that are getting increasingly close to a blow-up.
\end{abstract}

\section{Introduction}
% {\let\thefootnote\relax\footnote{{Implementation is available at \href{https://anonymous.4open.science/r/PINNs_Burgers-DC5E}{https://anonymized}}}}

%Scientific research has two primary purposes: understanding the first principles and solving practical problems. However, the task of finding first principles has been basically accomplished after establishing quantum mechanics, and the difficulty is only that the exact application of these laws leads to equations much too complicated to be soluble \cite{dirac1929}. Newtonian mechanics, gas dynamics, elastic theory, fluid mechanics, electromagnetism, and quantum mechanics provide the first principles for essentially all of natural science and engineering, and all of them are in the form of differential equations and mostly partial differential equations (PDEs). Therefore, solving PDEs is an important topic in computational science and engineering.
Partial Differential Equations (PDEs) are used for modeling a large variety of physical processes from fluid dynamics to bacterial growth to quantum behaviour at the atomic scale. But differential equations that can be solved in “closed form,” that is, by means of a formula for the unknown function, are the exception rather than the rule. Hence over the course of history, many techniques for solving PDEs have been developed. However, even the biggest of industries still find it extremely expensive to implement the numerical PDE solvers -- like airplane industries aiming to understand how wind turbulence pattern changes with changing aerofoil shapes, \cite{jameson2002using} need to choose very fine discretizations which can often increase the run-times prohibitively. 

In the recent past, deep learning has emerged as a competitive way to solve PDEs numerically. We note that the idea of using nets to solve PDEs dates back many decades \cite{Lagaris_1998} \cite{broomhead1988RBF}. In recent times this idea has gained significant momentum and ``AI for Science" \cite{karniadakis2021piml} has emerged as a distinctive direction of research. Some of the methods at play for solving PDEs neurally \cite{E_2021} are the Physics Informed Neural Networks (PINNs) paradigm \cite{raissi2019physics} \cite{lawal2022prototype}, ``Deep Ritz Method" (DRM, \cite{yu2018deepdrm}), ``Deep Galerkin Method" (DGM), \cite{sirignano2018dgm} and many further variations that have been developed of these ideas, \cite{kaiser2021datadriven, erichson2019physicsinformed, Wandel_2021, li2022learning, salvi2022neural}. An overarching principle that many of these implement is to try to constrain the loss function by using the residual of the PDE to be solved. 

% People \cite{Beck_2021, berner2020numerically} have met such a disaster in many cases like solving Kolmogorov PDEs.
% One such example topic in this field is climate forecast and some works from \cite{2021ComEE...2..159G, irrgang2021, Yuval_2020} have shown how AI techniques help the fundamental science. Furthermore, researches from \cite{kutyniok2020theoretical, gühring2019error, geist2020numerical} gave very theoretical analysis for solving PDEs using deep neural networks and also approximate the error bound when using the activation gate as a component of deep learning techniques. Nevertheless, all these contributions to solving PDEs using neural networks strongly supported the development in this field, inspiring the followers to go further to find a generic method.
% For a simple introduction to the PINNs, the principle behind it involves employing neural networks to approximate solutions of PDEs by minimizing a loss function. The loss function comprises residuals for initial and boundary conditions as well as partial differential equation residuals at selected points within the domain. The prototype idea of PINNs can be found in a series of meshless numerical methods \cite{lawal2022prototype}.
%Currently, there are a plethora of different ways in which one can write a loss function involving a partial differential equation (and its available approximate solutions) and neural net(s) \cite{karniadakis2021piml} s.t the net(s) obtained at the end of training can be used to infer approximate solutions to the intended PDEs. 

These different data-driven methods of solving the PDEs can broadly be classified into two kinds, {\bf (1)} ones which train a single neural net to solve a specific PDE and {\bf (2)} operator methods -- which train multiple nets in tandem to be able to solve a family of PDEs in one shot. \cite{Lu_2021,LU2022114778,wang2021pdeOnet} 
 The operator methods are particularly interesting when the underlying physics is not known and state-of-the-art approaches of this type can be seen in works like  \cite{mishra2023convolutional}. 
 
 For this work, we focus on the PINN formalism from \cite{raissi2019physics}. Many studies have demonstrated the success of this setup in simulating complex dynamical systems like the Navier-Stokes PDE \cite{ARTHURS2021110364, wang2020physicsinformed, Eivazi_2022}, the Euler PDE \cite{wang2022asymptotic}, descriptions of shallow water wave by the Korteweg-De Vries PDEs \cite{hu2022XPINNs} and many more.  

The existing literature in classical numerical analysis for estimating the error of approximating PDE solutions falls short in analyzing the performance of PINNs. In Section~\ref{sec:reltd:inapplicable}, we provide a review of this inapplicability of classical results to the specific case of Burgers' PDE that we focus on in this work. The work in \cite{siddhartha2022generror,de2022error} has provided the first-of-its-kind bounds on the generalization error of PINNs for approximating various standard PDEs, including the Navier-Stokes' PDE. Such bounds strongly motivate why minimization of the PDE residual at collocation points can be a meaningful way to solve the corresponding PDEs. However, the findings and analysis in \cite{krishnapriyan2021characterizing, wangperdiakaris2021failure} point out that the training dynamics of PINNs can be unstable and failure cases can be found among even simple PDE setups. It has also been studied that when trivial solutions can exist for the PDE, the PINN training can get stuck at those solutions \cite{rohrhofer2022role, wong2022pinn}. Work in \cite{wang2022respecting} has shown that traditional ways of solving PINNs can violate causality.
% \note{ Despite PINNs' excellent approximation power and generalization ability, a theoretical understanding of their convergence and generalization properties remains insufficiently explored. }
% In works like \cite{daw2022rethinking} some novel adjustments to the training process have also been suggested to mitigate some of these maladies.

However, in all the test cases above the target solutions have always been nice functions. But an interesting possibility with various differential equations representing dynamical systems is that their solution might have a finite-time blow-up. Blow-up is a phenomena where the solution $\vu$ becomes infinite at some points as time $t$ approaches a certain time  $T<\infty$, while the solution is well-defined for all $0<t<T$ i.e. 
\begin{equation*}
    \sup_{x \in D}\abs{\vu(\vx,t)}\rightarrow \infty \quad as \quad t\rightarrow T^-
\end{equation*}
% i.e the solution of the (P)DE develops a singularity after evolving for a finitely long amount of time whilst starting from very smooth initial conditions. 
One can see simple examples of this fascinating phenomenon, for example, for the following ODE $~\frac{du}{dt} = u^2, ~u(0) = u_0, ~u_0 > 0$ it's easy to see that it's solution blows-up at $t=\frac{1}{u_0}$. Wintner's theorem \cite{wintner1945existence} provided a sufficient condition for a very generic class of ODEs for the existence of a well-defined solution for them over the entire time-domain, in other words, the non-existence of a finite-time blowup. More sophisticated versions of such sufficient conditions for global ODE solutions were subsequently developed  \cite{cooke1955existence} and \cite{pazy1983existence} (Theorem 3.3). Non-existence of finite-time blow-ups have also been studied in control theory \cite{lin1996control} under the name of ``forward completeness" of a system.

% Osgood's uniqueness theorem for ODE \cite{osgood1898ODE} outlines the circumstances that establish when an ODE possesses an unique solution to the initial value problem. However, for PDEs, an analogous condition determining a unique solution does not exist. This further complicates the search for a blow-up scenario in a Cauchy problem because we cannot know in advance whether the solution obtained through a particular method will exhibit a blow-up or if it is stuck at a non-blowup solution.
% \note{add the ODE example! And explain the Osgood theorem here!}

The existence of a blow-up makes solving PDEs difficult to solve for classical approximation methods. There is a long-standing quest in numerical methods of PDE solving to be able to determine the occurrence, location and nature of finite time blow-ups \cite{stuart_floater_1990}. A much investigated case of blow-up in PDE is for the exponential reaction model $\vu_t = \Delta \vu + \lambda e^{\vu},\quad \lambda > 0$ which was motivated as a model of combustion under the name Frank-Kamenetsky equation. The nature of blow-up here depends on the choice of $\lambda$, the initial data and the domain. Another such classical example is $\vu_t = \Delta \vu + \vu^p$ and both these semi-linear equations were  studied in the seminal works \cite{fujita1966Ocauchy, fujita1969nonlinear} which pioneered systematic research into finite-time blow-ups of PDEs.
% \note{add all the classical references I had sent}

To the best of our knowledge, the behaviour PINNs in the proximity of finite-time blow-up has not received adequate attention in prior work on PINNs. We note that there are multiple real-world phenomena whose PDE models have finite-time blow-ups and these singularities are known to correspond to practically relevant processes -- such as in chemotaxis models \cite{miguel1997chemo,he2019chemosuppressing,chen2022asymptotically,tanaka2023chemofinite} and thermal-runoff models \cite{bebernes1981thermalmathematical,lacey1983thermalmathematical,dold1991thermalasymptotic,herrero1993thermalplane,lacey1995thermal}. 

In light of the recent rise of methods for PDE solving by neural nets, it raises a curiosity whether the new methods, in particular PINNs, can be used to reliably solve PDEs near such blow-ups. While a general answer to this is outside the scope of this work, {\em we derive theoretical risk bounds for PINNs which are amenable to be tested against certain analytically describable finite-time blow-ups. Additionally, we give experiments to demonstrate that our bounds retain non-trivial insights even when tested in the proximity of such singularities.}

In \cite{wang2022asymptotic}, the authors provided thought-provoking experimental evidence was given that PINNs could potentially discover PDE solutions with blow-up even when their explicit descriptions are not known. Since finite-time blow-ups are a phenomenon at a particular instant of time (and hence a measure zero set), there is a surprise that PINN methods can find neural surrogates sensitive to it while having been trained on averaged losses. Hence inspired, here we embark on a program to understand this emerging interface from a rigorous viewpoint and show how well the theoretical risk bounds correlate to their experimentally observed test errors - in certain blow-up situations. As our focus point, we will use reduced models of fluid dynamics, i.e Burgers' PDE in one and two spatial dimensions. The choice of our test case is motivated by the fact that these PDE setups have analytic solutions with blow-up -- as is necessary to do a controlled study of PINNs facing such a situation. We note that it is otherwise very rare to know exact fluid-like solutions which blow-up in finite-time \cite{tao2016finiteeuler, tao2016finitenavierstokes} 

\paragraph{Notation}
In the subsequent section we use $d+1$ to represent dimensions, here $d$ is the number of spatial dimensions and $1$ is always the temporal dimension.
Nabla $(\nabla)$ is used to represent the differential operator i.e. $(\frac{\partial}{\partial x_1}, .~.~., \frac{\partial}{\partial x_d})$. And for any real function $u$ on a domain $D$, $\norm{u(x)}_{L^\infty(D)}$ will represent $\sup_{x \in D} \abs{u(x)}$.

\section{Informal Summary of Our Results}

Firstly we give a brief review of the framework of Physics-Informed Neural Networks (PINNs) which is the focus of this work. Towards that consider the following specification of a PDE describing a dynamical system and satisfied by an appropriately smooth function $\vu(\vx,t)$ 
\begin{align}\label{eq:pinnreview.pde}
    &\vu_t + \gN_\vx[\vu] = 0,\quad \vx \in D, t \in [0,T] \nonumber\\
    &\vu(\vx,0) = h(\vx),\quad \vx \in D \\
    &\vu(\vx,t) = g(\vx,t),\quad \vt \in [0,T], \vx \in \partial D \nonumber
\end{align}
where $\vx$ and $t$ represent the space and time dimensions, subscripts denote the partial differentiation variables, $\gN_\vx[\vu]$ is the nonlinear differential operator, $D$ is a subset of $\R^d$ s.t it has a well-defined boundary $\partial D$. Following \cite{raissi2019physics}, we try to approximate $\vu(\vx,t)$ by a deep neural network $\vu_\theta(\vx,t)$, and then we can define the corresponding residuals as,
\begin{align*}
\gR_{pde}(x,t) \coloneqq \partial_t \vu_\theta + \gN_\vx[\vu_\theta(\vx,t)], ~\gR_t(x) \coloneqq \vu_\theta(\vx,0) - h(\vx), ~\gR_b(x,t) \coloneqq \vu_\theta(\vx,t) - g(\vx,t)
\end{align*}
Note that the partial derivative of the neural network ($\vu_\theta$) can be easily calculated using auto-differentiation \cite{baydin2018automatic}. The neural net is then trained on an empirical loss function,
\begin{align*}
    \gL(\theta) \coloneqq \gL_{pde}(\theta) + \gL_t(\theta) + \gL_b(\theta)
\end{align*}
where $\gL_{pde},~\gL_t$ and $\gL_b$ penalize for $\gR_{pde},~\gR_t$ and $\gR_b$ respectively for being non-zero. Typically it would take the form
\begin{align*}
    \gL_{pde} = \frac{1}{N_{pde}} \sum_{i=1}^{N_{pde}} \gR_{pde}(x_r^i,t_r^i)^2, ~\gL_{t} = \frac{1}{N_{t}} \sum_{i=1}^{N_{t}} \gR_{t}(x_t^i)^2, ~\gL_{b} = \frac{1}{N_{b}} \sum_{i=1}^{N_{b}} \gR_{b}(x_b^i,t_b^i)^2
\end{align*}
where $(x_r^i,t_r^i)$ denotes the collocation points, $(x_t^i)$ are the points sampled on the spatial domain for the initial loss and $(x_b^i,t_b^i)$ are the points sampled on the boundary for the boundary loss. The aim here is to train a neural net $\vu_\theta$ such that $\gL_\theta$ is as close to zero as possible.

%\subsection{Informal Summary of Our Results}

At the very outset, we note that to the best of our knowledge there are no available off-the-shelf generalization bounds for any setup of PDE solving by neural nets where the assumptions being made include any known analytic solution with blow-up for the corresponding PDE. So, as a primary step we derive new risk bounds for Burgers's PDE in Theorem~\ref{th:ndburgers.error_upperbound} and Theorem~\ref{th:int_burger_gen_error_bound}, where viscosity is set to zero and its boundary conditions are consistent with finite-time blow-up cases of Burgers' PDE that we eventually want to test on. We note that despite being designed to cater to blow-up situations, the bound in Theorem~\ref{th:int_burger_gen_error_bound} is also ``stable'' in the sense of \cite{wang2022stability}.

Our experiments reveal that for our test case with Burgers' PDE the theoretical error bounds that we derive \footnote{It is to be noted that it is routine for analytic neural net generalization bounds to be vacuous, particularly for deep nets as considered in the experiments here.}, are such that they do maintain a non-trivial amount of correlation with the $L^2$-distance of the derived solution from the true solution. The plot in \twofigref{fig:plot_Burgers_loss}{fig:plot_ndburgers} vividly exhibit the presence of this strong correlation between the derived bounds and the true risk despite the experiments being progressively made on time domains such that the true solution is getting arbitrarily close to becoming singular. We will also show that for the one-dimensional blow-up case that we consider, the time to solve the neural surrogate is almost independent of the proximity to the blow-up that we want to solve it to and also that the derived bounds drop with the width of the net and hence reflecting the benefits of using overparameterized nets. 

% \note{refer to the correlation figures above}

A key feature of our approach to this investigation is that we do not tailor our theory (Theorem~\ref{th:ndburgers.error_upperbound}) to the experimental setups we test on later. We posit that this is a fair way to evaluate the reach of PINN theory whereby the theory is built such that it caters to any neural net and any solution of the PDE while these generically derived bounds get tested on the hard instances. \footnote{One can surmise that it might be possible to build better theory exploiting information about the blow-ups - like if the temporal location of the blow-up is known. However, it is to be noted that building theory while assuming knowledge of the location of the blow-up might be deemed unrealistic given the real-world motivations for such phenomena.}

\section{Related Works}\label{sec:review}
 To the best of our knowledge the most general population risk bound for PINNs has been proven in \cite{karniadakis2022xpinn}, and this result applies to all linear second order PDE and it is a  Rademacher complexity based bound. This bound cannot be applied to our study since Burgers' PDE is not a linear PDE. The authors in \cite{siddhartha2022generror} derived generalization bounds for PINNs, that unlike \cite{karniadakis2022xpinn}, explicitly depend on the trained neural net. They performed the analysis for several PDEs, and the ``viscous scalar conservation law" being one of them, which includes the $1+1$-Burgers' PDE. However, for testing against analytic blow-up solutions, we need such bounds at zero viscosity unlike what is considered therein, and most critically, unlike \cite{siddhartha2022generror} we keep track of the prediction error at the spatial boundary of the computational domain with respect to non-trivial functional constraints. 
 
The authors in \cite{de2022error} derived a generalization bound for Navier-Stokes PDE, which too depends on the trained neural net. We note that, in contrast to the approach presented in \cite{de2022error}, our method does not rely on the assumption of periodicity in boundary conditions or divergencelessness of the true solution. These flexibilities in our setup ensure that our bound applies to known analytic cases of finite-time blow-ups for the $d+1$-Burgers' PDE.
 
 % In this they have defined the neural network in a \emph{tree-like} function space i.e. Barron space, which lets them get a prior generalization bound that depends on the complexity of the target function which is very well represented by the Barron norm. For the posterior bound, the spectral norm of the trained neural net is being used to incorporate the complexity of the neural net. The Barron norm of the target function is very closely correlated to the spectral norm since a more complex\footnote{higher Barron norm} target function would mean the trained neural net is more complex and therefore a higher spectral norm. These bounds cannot be directly applied for our case since Burgers' is not a linear PDE.

Notwithstanding the increasing examples of the success of PINNs, it is known that PINNs can at times fail to converge to the correct solution even for basic PDEs -- as reflected in several recent studies on characterizing the “failure modes” of PINNs. Studies reported in \cite{wangperdiakaris2021failure}, and more recently in \cite{dawperdiakaris2023_failure} have demonstrated that sometimes this failure can be attributed to problems associated with the loss function, specifically the uneven distribution of gradients across various components of the PINN loss. The authors in \cite{wangperdiakaris2021failure} attempt to address this issue by assigning specific weights to certain parts of the loss function. While \cite{daw2022rethinking} developed a way to preferentially sample collocation points with high loss and subsequently use them for training. In \cite{krishnapriyan2021characterizing} a similar issue with the structure of the loss function was observed. While not changing the PINN loss function, they introduced two techniques: ``curriculum regularization" and ``sequence-to-sequence learning" for PINNs to enhance their performance. In \cite{wangperdiakaris2022ntkfailure} PINNs have been analyzed from a neural tangent kernel perspective to suggest that PINNs suffer from ``spectral-bias"\cite{rahaman2019spectral} which makes it more susceptible to failure in the presence of ``high frequency features" in the target function. They propose a method for improving training by assigning weights to individual components of the loss functions, aiming to mitigate the uneven convergence rates among the various loss elements.

\subsection{Inapplicability of Classical Numerical Analysis Error bounds to PINN Experiments} \label{sec:reltd:inapplicable}

To put the ongoing attempts at theoretical analysis of PINNs in perspective, it is to be noted that to the best of our knowledge existing results in numerical analysis cannot be deployed to understand PINN training - as has been the target here. Specifically for the zero viscosity Burgers' PDE one can see that in works like \cite{johnson1987}, the theory does not seem to be give a bound on the distance from the true solution of the solution found by the finite element method.

More generally, in works such as Corollary 3.5 in \cite{tadmur91} the authors consider a weak solution of the $\varepsilon$-viscosity regularized Burgers' PDE and derive bounds on the local $L^p$-distance between the weak solution and the true solution at zero viscosity. There is no obvious way to apply these bounds for a PINN solution since the trained net has no guarantee to be satisfying the conditions required of the surrogate here. 

With results like Theorem 2.1 in \cite{tadmur92}, we observe that these too don't have an obvious way for the bounds to be applied for PINN experiments because they need stringent conditions (like satisfying the conservativeness property) to be true for the approximant and there is no natural way to know if the neural surrogate satisfies these conditions. Also both the above cited classical bounds are not tailored to any compact domain and hence there is no boundary condition error that is getting tracked there, as in our Theorem \ref{th:int_burger_gen_error_bound}.

\section{Results}\label{sec:main_results}

In the next two subsections, we will present the main generalization bounds that we prove for Burgers' PDE being solved by a neural surrogate. Next, we will experimentally demonstrate the high correlation of these bounds to measured test error when neural nets are trained to solve for certain exact Burgers' PDE solutions which have a finite-time blow-up, in one and two spatial dimensions.

\subsection{Generalization Bounds for the $(d+1)$-Dimensional Burgers' PDE}
The PDE that we consider is as follows,
\begin{align}
    \nonumber \partial_t \vu + (\vu \cdot \nabla) \vu = 0 \\
    \vu (t = t_0) = \vu_{t_0} \label{ndburgers.2}
\end{align}
Here $\vu : D \times [t_0, T] \rightarrow \R^d$ is the fluid velocity and $\vu_{t_0} : D \rightarrow \R^d$ is the initial velocity. Then corresponding to a surrogate solution $\vu_\theta$ we define the residuals as,
\begin{align}
    \gR_{{pde}} \coloneqq \partial_t \vu_\theta + (\vu_\theta \cdot \nabla) \vu_\theta \label{ndburgers.rpde}\\
    \gR_{{t}} \coloneqq \vu_\theta(t = t_0) - \vu(t=t_0)\label{ndburgers.rt}
\end{align}
Corresponding to the true solution $\vu$, we will define the $L^2$-~risk of any surrogate solution $\vu_\theta$ as,
\begin{align*}
    \int_\Omega \norm{\vu(\vx, t) - \vu_\theta(\vx, t)}_2^2 \dd{\vx}\dd{t}
\end{align*}
In the following theorem we consider $t_0 = \frac{-1}{\sqrt{2}} + \delta$ and $T = \delta$ for some $\delta > 0$. Here the spatial domain is represented by $D \subset \R^d$ and $\Omega$ represents the whole domain $D \times [t_0, T]$.

%\note{Are you sure you need this specific D in the proof?}

\begin{theorem}\label{th:ndburgers.error_upperbound}
    Let $d \in \N$ and $\vu \in C^1(D \times [t_0, T] )$ be the unique solution of the (d+1)-dimensional Burgers' equation given in equation \ref{ndburgers.2}. Then for any $C^1$ surrogate solution to equation \ref{ndburgers.2}, say $\vu_\theta$, the $L^2$-risk with respect to the true solution is bounded as,
    \begin{align}
        % \int_\Omega \norm{\vu(\vx, t) - \vu_\theta(\vx, t)}_2^2 \dd{\vx}\dd{t} &\leq C_2 \left[ \frac{-1}{\sqrt{2}} +  \frac{C_1}{4} e^{\frac{C_1}{\sqrt{2}}} \right] \nonumber\\
         \log\left( \int_\Omega \norm{\vu(\vx, t) - \vu_\theta(\vx, t)}_2^2 \dd{\vx}\dd{t}\right) &\leq \log{\left(\frac{C_1 C_2}{4}\right)} + \frac{C_1}{\sqrt{2}} \label{ndburgers.upperbound}
    \end{align}
    where,
    \begin{align*}
        C_1 &= d^2 \norm{\nabla \vu_\theta}_{L^\infty(\Omega)} \nonumber\\
        &+ 1 + d^2 \norm{\nabla \vu}_{L^\infty(\Omega)}\\
        C_2 &= \int_D \norm{\gR_t}_2^2 \dd{\vx} + \int_{\Omega} \norm{\gR_{pde}}_2^2  \dd{\vx} \dd{t} + d^2 \norm{\nabla \vu_\theta}_{L^\infty(\Omega)} \int_{\Omega} \norm{\vu_\theta}_2^2 \dd{\vx} \dd{t} \\
        &+ d^2 \norm{\nabla \vu}_{L^\infty(\Omega)} \int_{\Omega} \norm{\vu}_2^2 \dd{\vx} \dd{t}
    \end{align*}
    % \vspace{-3.00em}
\end{theorem}
The theorem above has been proved in \ref{subsec:ndburgers.proof} We note that the bound presented in equation \ref{ndburgers.upperbound} does not make any assumptions about the existence of a blow-up in the solution and its applicable to all solutions that have continuous first derivatives however large, as would be true for the situations very close to blow-up as we would consider. Also, we note that the bound in \cite{de2022error} makes assumptions (as was reviewed in Section \ref{sec:review}) which (even if set to zero pressure) prevent it from being directly applicable to the setup above which can capture analytic solutions arbitrarily close to finite-time blow-up.

{\em Secondly,} note that these bounds are not dependent on the details of the loss function that might eventually be used in the training to obtained the $\vu_\theta$. In that sense such a bound is more universal than usual generalization bounds which depend on the loss.  

{\em Lastly,} note that the inequality proven in Theorem \ref{th:ndburgers.error_upperbound} bounds the distance of the true solution from a PINN solution in terms of (a) norms of the true solution and (b) various integrals of the found solution like its norms and  unsupervised risks on the computation domain. Hence this is not like usual generalization bounds that get proven in deep-learning theory  literature where the LHS is the population risk and RHS is upperbounding that by a function that is entirely computable from knowing the training data and the trained net. 

Being in the setup of solving PDEs via nets lets us contruct such new kinds of bounds which can exploit knowledge of the true PDE solution.

While Theorem \ref{th:ndburgers.error_upperbound} is applicable to Burgers' equations in any dimensions, it becomes computationally very expensive to compute the bound in higher dimensions. Therefore, in order to better our intuitive understanding, we separately analyze the case of $d=1$, in the upcoming \Secref{subsec:1d-burgers}. Furthermore, the RHS of \plaineqref{ndburgers.upperbound} only sees the errors at the initial time and in the space-time bulk. In general dimensions it is rather complicated to demonstrate that being able to measure the boundary risks of the surrogate solution can be leveraged to get stronger generalization bounds.  But this can be transparently kept track of in the $d=1$ case - as we will demonstrate now for a specific case with finite-time blow-up. Along the way, it will also be demonstrated that the bounds possible in one dimension - are ``stable'' in a precise sense as will be explained after the following theorem. 

% \note{ note that the RHS above only sees the errors at the intial time and the bulk - explain that its not easy to see in the d dimensional setup as to how one might leverage keeping track of the prediction errors at the spatial boundaries. - but this can be transparently kept track of in the 1+1 case}

% \note{
% \begin{enumerate}
%     \item point out that this bound is independent of anything about what the true solution is whether it has a blow-up or not
%     \item point out if it can follow or not just by putting viscosity = 0 in Siddhartha's bound. Does it? If not then emphasize that!
%     \item then say that its clear that computationally this bound can be very expensive    and hence for getting better intuition we try to get some specific bounds tuned to the 1+1 special case 
% \end{enumerate}
% }
%%%%%%%%%%%%%%%%%%%%%%%%%%%%%
% subsection : The 1+1 Burgers
%%%%%%%%%%%%%%%%%%%%%%%%%%%%%
% \subsection{Analysis of Burgers' PDE in $1+1$ Dimensions With Zero Viscosity \& Finite-Time Blow-Up}
\subsection{Generalization Bounds for a Finite-Time Blow-Up Scenario with (1+1)-Dimensional Burgers' PDE}
\label{subsec:1d-burgers}
% \note{Write a 1+1 burgers without the $\delta$ dependency.}
% The PDE that we consider is
% \begin{align}
%     u_t + uu_{x} &= 0 \label{eq:burger_pdes:1}\\
%     u(x,t_0) &= \frac{x}{-1+t_0} \label{eq:burger_pdes:2}\\
%     u(-1,t) = \frac{1}{1-t}\ &;\quad  u(1,t) = \frac{1}{t-1} \label{eq:burger_pdes:3}
% \end{align}
For $u : [-1, 1] \times [t_0, T] \rightarrow \R$ being at least once continuously differentiable in each of its variables we consider a Burgers's PDE as follows on the space domain being $[-1,1]$ and the two limits of the time domain being specified as $t_0 = -1 + \delta$ and $T = \delta$ for any $\delta > 0$,
\begin{align}
    \nonumber u_t + uu_{x} &= 0 \\
    \nonumber u(x,-1+\delta) &= \frac{x}{-2+\delta} \\
    u(-1,t) = \frac{1}{1-t} ~&;~   u(1,t) = \frac{1}{t-1} \label{eq:burger_pdes:6}
\end{align}
We note that in the setup for Burger's PDE being solved by neural nets that was analyzed in the pioneering work in \cite{siddhartha2022generror}, the same amount of information was assumed to be known i.e the PDE, an initial condition and boundary conditions at the spatial boundaries. However in here, the values we choose for the above constraints are non-trivial and designed to cater to a known solution for this PDE i.e $u = \frac{x}{t-1}$ which blows up at $t=1$.

For any $C^1$ surrogate solution to the above, say $u_\theta$ its residuals can be written as,
\begin{align}
    \gR_{int,\theta}(x,t) &\coloneqq \partial_t (u_\theta(x,t)) +  \partial_x \frac{u^2_\theta(x,t)}{2} \label{eq:burger_pdes_residual:1}\\
    \gR_{tb,\theta}(x) &\coloneqq u_\theta(x,-1+\delta) - \frac{x}{-2+\delta} \label{eq:burger_pdes_residual:2}\\
    (\gR_{sb,-1,\theta}(t), \gR_{sb,1,\theta}(t)) &\coloneqq \left(u_\theta(-1,t) - \frac{1}{1-t},\ u_\theta(1,t) - \frac{1}{t-1} \right) \label{eq:burger_pdes_residual:3}
\end{align}
% \note{define $u_\theta$ above!}
We define the $L^2-$risk of $u_\theta$ with respect to the true solution $u$ of equation \ref{eq:burger_pdes:6} as,
\begin{align}
    \mathcal{E}_G (u_\theta) \coloneqq \left(\int_{-1+\delta}^\delta \int_{-1}^1 |u(x,t) - u_\theta(x,t)|^2\ dx dt \right)^\frac{1}{2}
    \label{eq:gen_error}
\end{align}

\begin{theorem}\label{th:int_burger_gen_error_bound}
    Let $u \in C^1((-1+\delta,\delta) \times (-1,1))$ be the unique solution of the one dimensional Burgers' PDE in equation \ref{eq:burger_pdes:6}. Then for any surrogate solution for the same PDE, say $u^* \coloneqq u_{\theta^*}$ its risk as defined in  equation \ref{eq:gen_error} is bounded as,
    \begin{align}
        \gE_G^2\ &\leq\ \left[ 1 + Ce^{C} \right] \left[\int_{-1}^1 \gR_{tb,\theta^*}(x) dx + 2 C_{2b} \left(\int_{-1+\delta}^{\delta}\gR_{sb,-1,\theta^*}^2 (t) dt + \int_{-1+\delta}^{\delta}\gR_{sb,1,\theta^*}^2 (t) dt \right) \right. \nonumber\\ 
        &+ \left. 2 C_{1b} \left(\left(\int_{-1+\delta}^{\delta} \gR^2_{sb,-1,\theta^*}(t) dt\right)^\frac{1}{2} + \left(\int_{-1+\delta}^{\delta} \gR^2_{sb,1,\theta^*}(t) dt\right)^\frac{1}{2} \right) + \int_{-1+\delta}^{\delta} \int_{-1}^1 \gR_{int,\theta^*}^2(x,t) dx dt \right]
        \label{eq:int_burger_bound_th_1}
    \end{align}
    where $C = 1 + 2C_{u_x}$, with $C_{u_x} = \norm{u_x}_{L^\infty((-1+\delta,\delta) \times (-1,1))} = \norm{\frac{1}{t-1}}_{L^\infty([-1+\delta,\delta])} = \frac{1}{1-\delta}$ and 
    \begin{align}
       \nonumber  C_{1b} &= \norm{u(1,t)}^2_{L^\infty([-1+\delta,\delta])} = \norm{\frac{1}{1-t}}^2_{L^\infty([-1+\delta,\delta])} = \frac{1}{(1-\delta)^2} \\
        C_{2b} &= \norm{u_{\theta^*}(1,t)}_{L^\infty([-1+\delta,\delta])} + \frac{3}{2}\norm{\frac{1}{t-1}}_{L^\infty([-1+\delta,\delta])} = \norm{u_{\theta^*}(1,t)}_{L^\infty([-1+\delta,\delta])} + \frac{3}{2}\left( \frac{1}{1-\delta} \right)
        \label{eq:int_burger_bound__th_2}
    \end{align}
\end{theorem}
The theorem above has been proved in \ref{proof:int_burger_gen_error_bound}. Note that the RHS of \eqref{eq:int_burger_bound_th_1} is evaluable without exactly knowing the exact true solution $u$ -- the constants in \eqref{eq:int_burger_bound_th_1} only require some knowledge of the supremum value of $u$ at the spatial boundaries and the behaviour of the first order partial derivatives of $u$. 

Note that the most natural PINN risk function we can minimize (and what will be used in the experiments) is,
\begin{align}\label{eq:1drisk}
\gR_{1+1} \coloneqq \mathbb{E}[\abs{\gR_{int,\theta}(x,t)}^2] + \mathbb{E}[\gR_{tb,\theta}^2] + \mathbb{E}[\gR_{sb,-1,\theta}^2] + \mathbb{E}[\gR_{sb,1,\theta}^2]
\end{align}
The expectations in above are understood to be over separate distributions as appropriate for each of the terms. In light of this, most importantly, Theorem~\ref{th:int_burger_gen_error_bound} shows that despite the setting here being of proximity to finite-time blow-up, the naturally motivated PINN risk as stated above is ``$(L_2,L_2,L_2,L_2)$-stable'' \footnote{Suppose $Z_1, Z_2, Z_3, Z_4$ are four Banach spaces, a PDE defined by \plaineqref{eq:pinnreview.pde} is $Z_1, Z_2, Z_3, Z_4$-stable, if $\norm{\vu_\theta(x,t) - \vu(x,t)}_{Z4} = \gO(\norm{\partial_t \vu_\theta + \gN_\vx[\vu_\theta(\vx,t)]}_{Z_1} + \norm{\vu_\theta(\vx,0) - h(\vx)}_{Z_2} + \norm{\vu_\theta(\vx,t) - g(\vx,t)}_{Z_3})$ as $\norm{\partial_t \vu_\theta + \gN_\vx[\vu_\theta(\vx,t)]}_{Z_1}, \norm{\vu_\theta(\vx,0) - h(\vx)}_{Z_2}, \norm{\vu_\theta(\vx,t) - g(\vx,t)}_{Z_3} \rightarrow 0$ for any $\vu_\theta$} in the precise sense as defined in \cite{wang2022stability}. This stability property being true implies that if the PINN risk of the solution obtained is measured to be $\gO(\epsilon)$ then it would directly imply that the $L_2$-risk with respect to the true solution (\ref{eq:gen_error}) is also $\gO(\epsilon)$. And this would be determinable {\em without having to know the true solution at test time.}

% \gR_{pde}(x,t) \coloneqq \partial_t \vu_\theta + \gN_\vx[\vu_\theta(\vx,t)], ~\gR_t(x) \coloneqq \vu_\theta(\vx,0) - h(\vx), ~\gR_b(x,t) \coloneqq \vu_\theta(\vx,t) - g(\vx,t)
% \note{In \cite{wang2022stability}, they define a stability condition for PINN losses. For a certain PDE, if the PINN loss function (such as the $\ell_2$-loss) is not stable then according to \cite{wang2022stability}, the loss function is not ideal for training a PINN. Incidentally the PDE in \cite{de2022error} is $L_2,L_2,L_2$-stable, apparent from the upper-bound to the $\ell_2$-error. Our analysis for the $1+1$-case also shows that it is $L_2,L_2,L_2$-stable even though we do not have a periodic boundary condition or zero-divergence. But it remains an open question if the $d+1$-dimensional Burgers' is stable under $\ell_2$-loss.}

%We note that, when compared to \cite{siddhartha2022generror}, our PDE has viscosity zero but even then the analysis remains quite complex due to the presence of non-trivial boundary conditions. 
In \ref{appendix:subsec:burger_gen_error_bound} we apply quadrature rules on \plaineqref{eq:int_burger_bound_th_1} and show a version of the above bound which makes the sample size dependency of the bound more explicit. 

% \note{refer to the quadrature section here}
% \note{
% Below the theorem you need to explain a few things as to why is this above bound interesting. 
% \begin{itemize}
% \item firstly explain that the RHS of the above is evaluatable from knowing the trained net and the boundary and the initial conditions. The true solution or the details of the training data are NOT needed to know the RHS. 

% \item Our analysis is inspired from \cite{siddhartha2022generror} and we obtained Theorem \ref{th:int_burger_gen_error_bound} as the generalization bound for our setup.

% We note that our result does not follow as the straightforward 0-viscosity limit of the analysis in \cite{siddhartha2022generror}.  We show that, in contrast to them, having to track non-trivial boundary conditions makes our analysis as involved.

% \end{itemize}
% }

% The loss function can be written as
% \begin{align}
% \small
%     \gL(\theta) &\coloneqq \frac{1}{N_{tb}}\sum_{n=1}^{N_{tb}} w_n^{tb}|\mathcal{R}_{tb,\theta}(x_n)|^2 
%     + \frac{1}{N_{sb}}\sum_{n=1}^{N_{sb}} w_n^{sb}|\mathcal{R}_{sb,-1,\theta}(t_{n,\delta})|^2 \nonumber\\ 
%     &+ \frac{1}{N_{sb}}\sum_{n=1}^{N_{sb}} w_n^{sb}|\mathcal{R}_{sb,1,\theta}(t_{n,\delta})|^2
%     + \frac{\lambda}{N_{int}} \sum_{n=1}^{N_{int}} w_n^{int}|\mathcal{R}_{int,\theta}(x_n, t_{n,\delta})|^2
% \end{align}

% \note{what is the difference between the equation at the top and the bottom?}

%\note{use explicit data sum forms in stating the above theorem}
\subsection{Experiments}
Our experiments are designed to demonstrate the efficacy of the generalization error bounds that we presented above. The novelty of our experimental setup can be seen in the light of the brief overview we give below of how demonstrations of deep-learning generalization bounds have been done in the recent past. %We analyse the perturbation in the bound with changing net width, while keeping the data and depth constant. 

In the thought-provoking paper \cite{dan2017computing} the authors computed their bounds for $2-$layer neural nets at various widths to show the non-vacuous nature of their bounds. However, these bounds do not apply to any single neural net but to an expected neural net sampled from a specified distribution. Inspired by these experiments, works like \cite{neyshabur2017pac} and \cite{mukherjee2020study} perform a de-randomized PAC-Bayes analysis on the generalization error of neural nets - which can be evaluated on any given net.  

In works such as \cite{behnam2018towards} we see a bound based on Rademacher analysis of the generalization error and the experiments were performed for depth-2 nets at different widths to show the decreasing nature of their bound with increasing width -- a very rare property to be true for uniform convergence based bounds. It is important to point out that the training data is kept fixed while changing the width of the neural net in the setups in \cite{dan2017computing} and \cite{behnam2018towards}.

In \cite{behnam_arora2018} the authors instantiated a way to do compression of nets and computed the bounds on a compressed version of the original net. More recently in \cite{ramachandran2023_generror} the authors incorporated the sparsity of a neural net alongside the PAC-Bayes analysis to get a better bound for the generalization error. In their experiments, they vary the data size while keeping the neural net fixed and fortuitously the bound becomes non-vacuous for a certain width of the net. 

In this work, we investigate if theory can capture the performance of PINNs near a finite-time blow-up and if larger neural nets can better capture the nature of generalization error close to the blow-up. To this end, in contrast to the previous literature cited above, we keep the neural net fixed and vary the domain of the PDE. More specifically, progressively we choose time-domains arbitrarily close to the finite-time blow-up and test the theory at that difficult edge.  

\subsubsection{The Finite-Time Blow-Up Case of (1+1)-Dimensional Burgers' PDE from Section \ref{subsec:1d-burgers}}\label{sec:exp.burgers}
The neural networks we use here have a depth of $6$ layers, and we experiment at two distinct uniform widths of $30$ and $300$ neurons and the training loss is the empirical form of equation \ref{eq:1drisk}. For training, we use full-batch Adam optimizer for $100,000$ iterations and a learning rate of $10^{-4}$. We subsequently select the model with the lowest training error for further analysis. In Figure~\ref{fig:plot3d_Burgers}  the plots have been shown for the predicted and actual solution, for neural nets solving equation \ref{eq:burger_pdes:6} at different values of the $\delta$ parameter and it is clear that the visual resemblance of the neurally derived solution persists even quite close to the blowup at $\delta = 1$.

\begin{figure}[htbp!]
\centering
    \begin{subfigure}[h]{0.83\textwidth}
        \includegraphics[width=\textwidth]{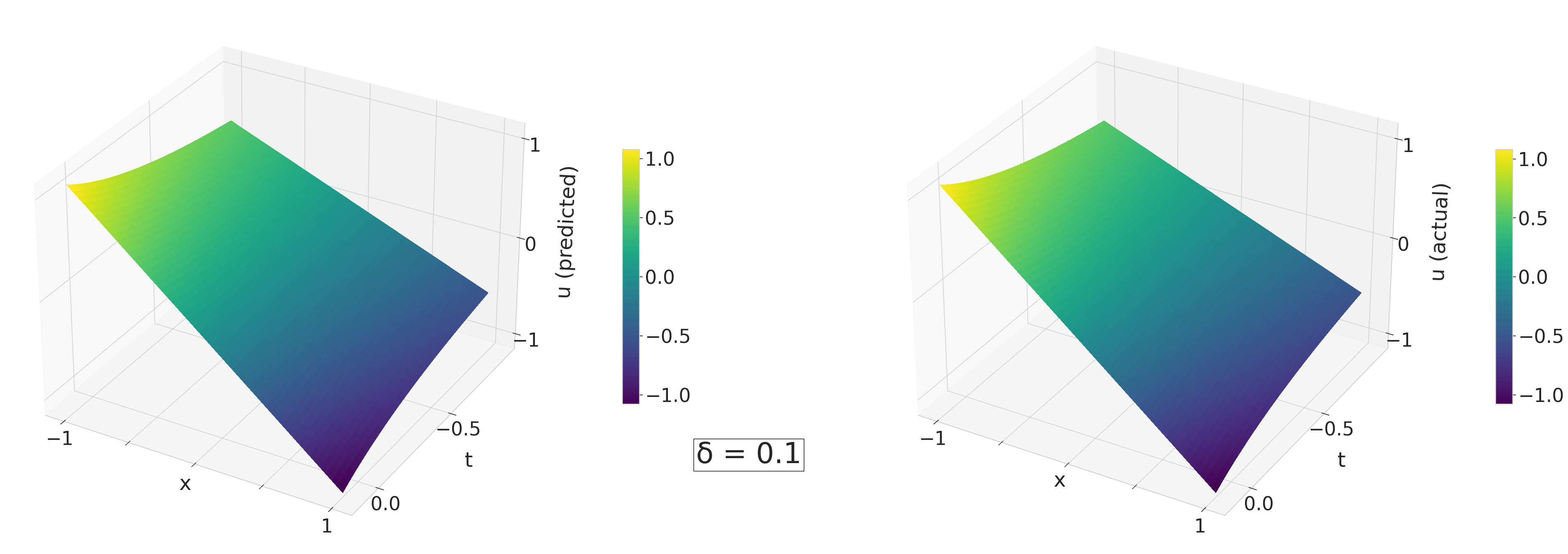}
        % \caption{$\delta$ = $0.1$}
    \label{fig:plot3d_RHSvsLHS_width300_delta0.1}
    \end{subfigure}
    \hfill
    \begin{subfigure}[h]{0.83\textwidth}
        \includegraphics[width=\textwidth]{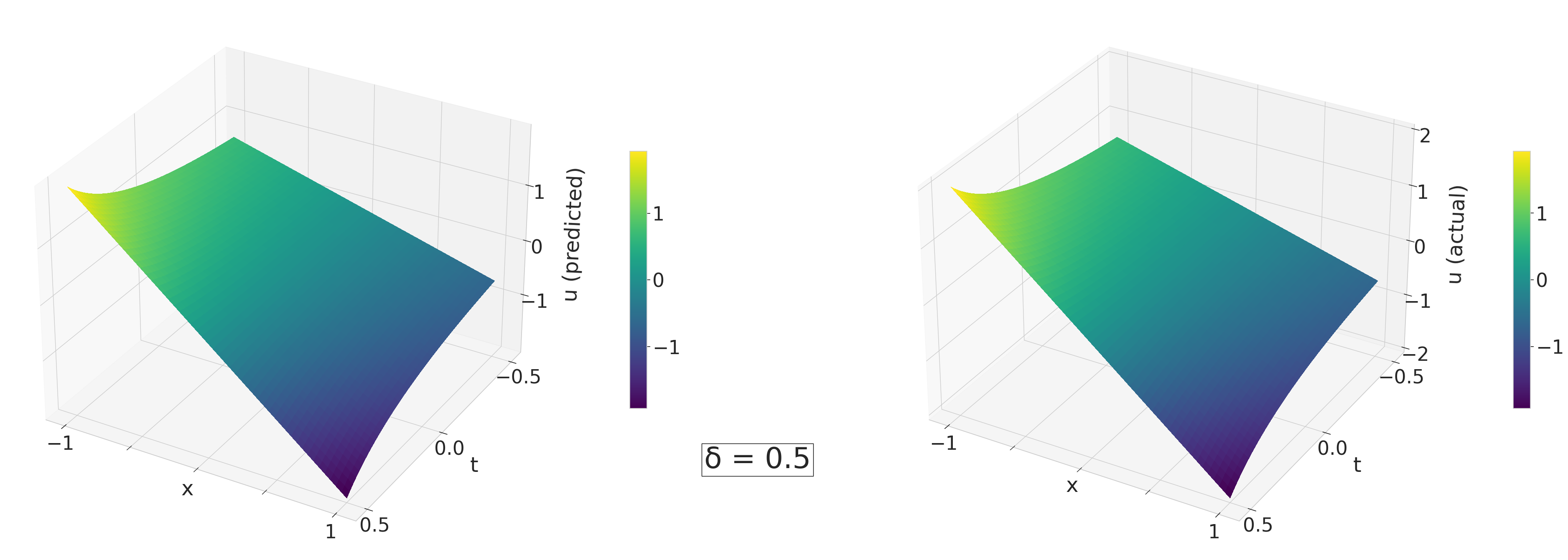}
        % \caption{$\delta$ = $0.5$}
    \label{fig:plot3d_RHSvsLHS_width300_delta0.5}
    \end{subfigure}
    \hfill
    \begin{subfigure}[h]{0.83\textwidth}
        \includegraphics[width=\textwidth]{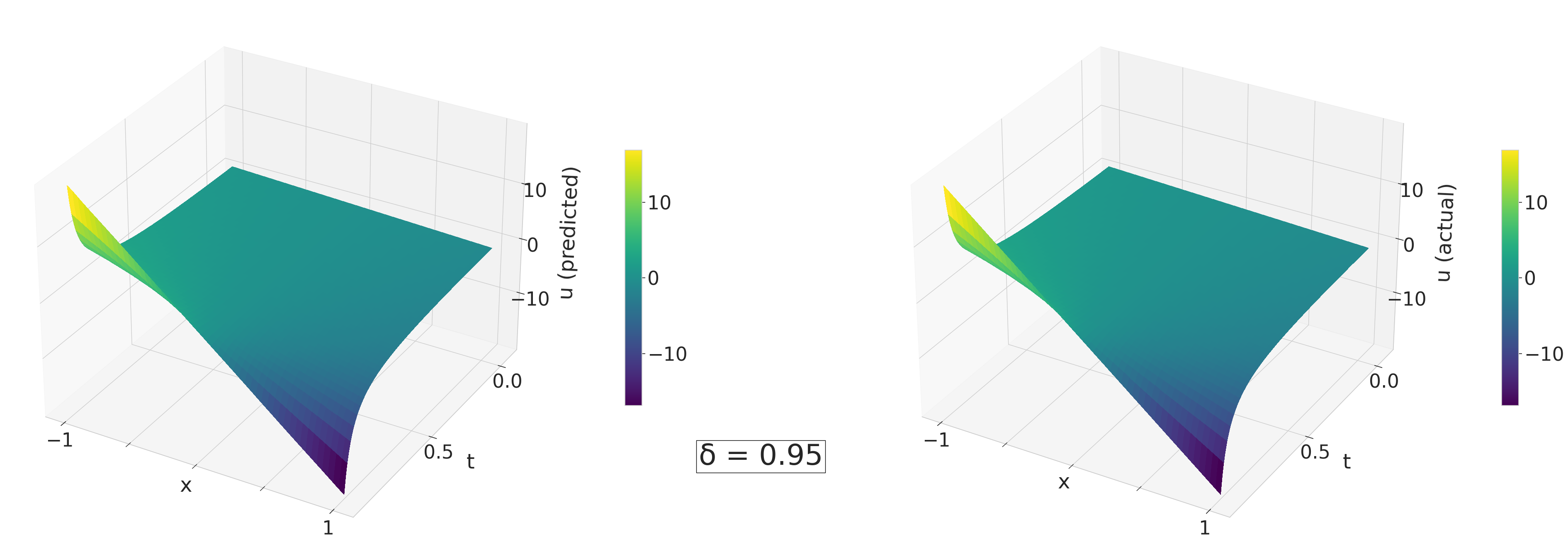}
        % \caption{$\delta$ = $0.950$}
    \label{fig:plot3d_RHSvsLHS_width300_delta0.95}
    \end{subfigure}
    \hfill
    \begin{subfigure}[h]{0.83\textwidth}
        \includegraphics[width=\textwidth]{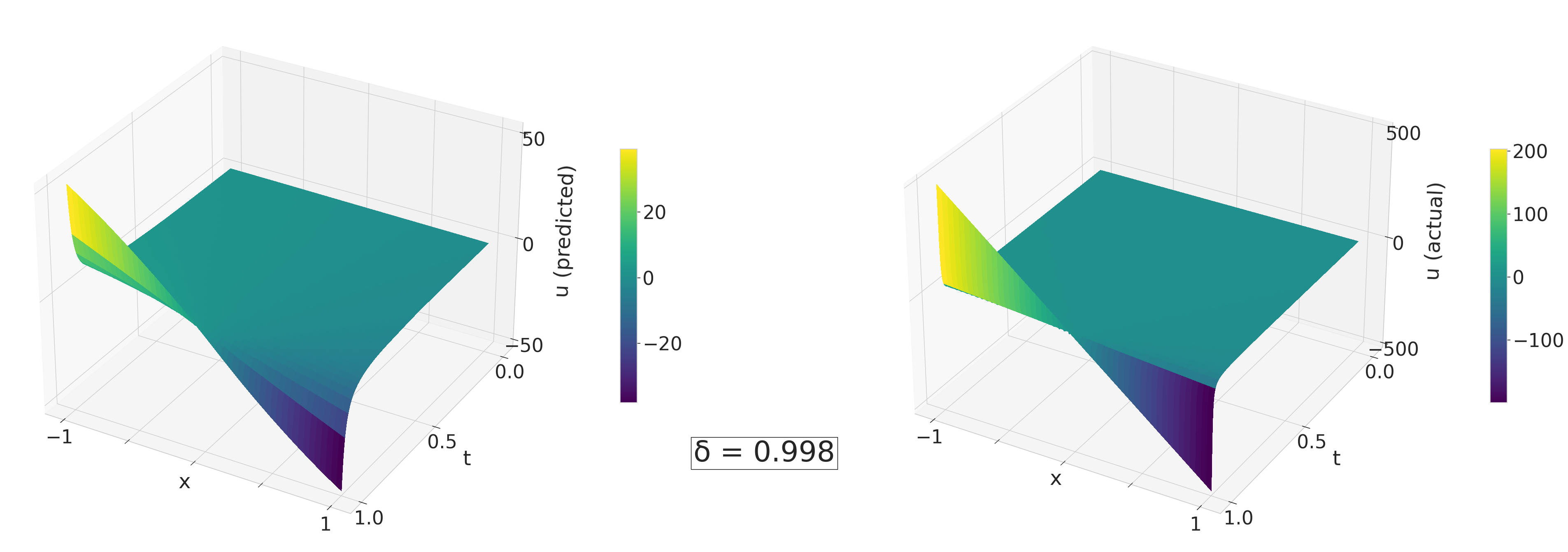}
        % \caption{$\delta$ = $0.998$}
    \label{fig:plot3d_RHSvsLHS_width300_delta0.998}
    \end{subfigure}
    \vspace{-1.00em}
    \caption{\small A demonstration of the visual resemblance between the neurally derived solution for equation \ref{eq:burger_pdes:6} (left) and the true solution (right) at different values of the $\delta$ parameter getting close to the PDE with blow-up at $\delta=1$. A PINN with a width of $300$ and a depth of $6$ was trained to generate the plots on the left.}
    % \vspace{-0.80em}
    \label{fig:plot3d_Burgers}
\end{figure}

\begin{figure}[htbp!]
\centering
    \begin{subfigure}[h]{0.49\textwidth}
        \includegraphics[width=\textwidth]{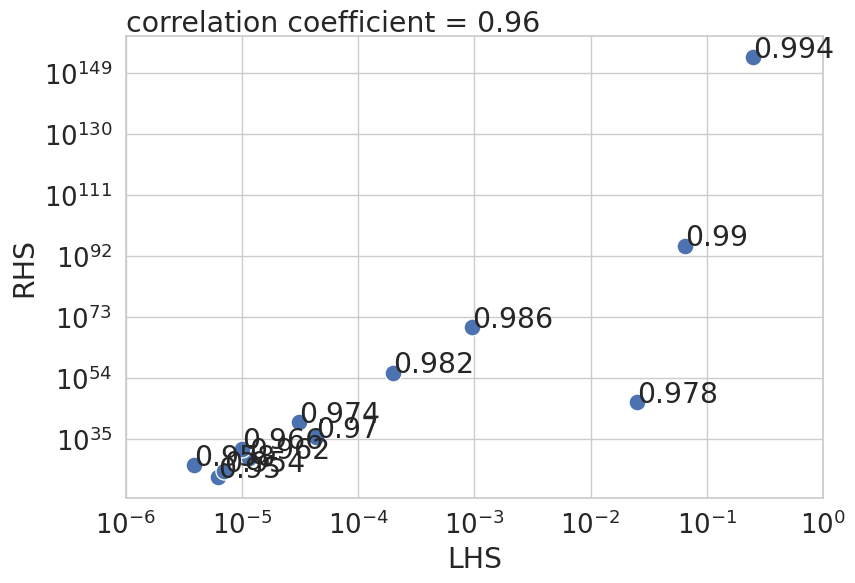}
        \vspace{-1.40em}
        \caption{width=30}
        \label{fig:plot_RHSvsLHS_width30}
    \end{subfigure}
    \hfill
    \begin{subfigure}[h]{0.49\textwidth}
        \includegraphics[width=\textwidth]{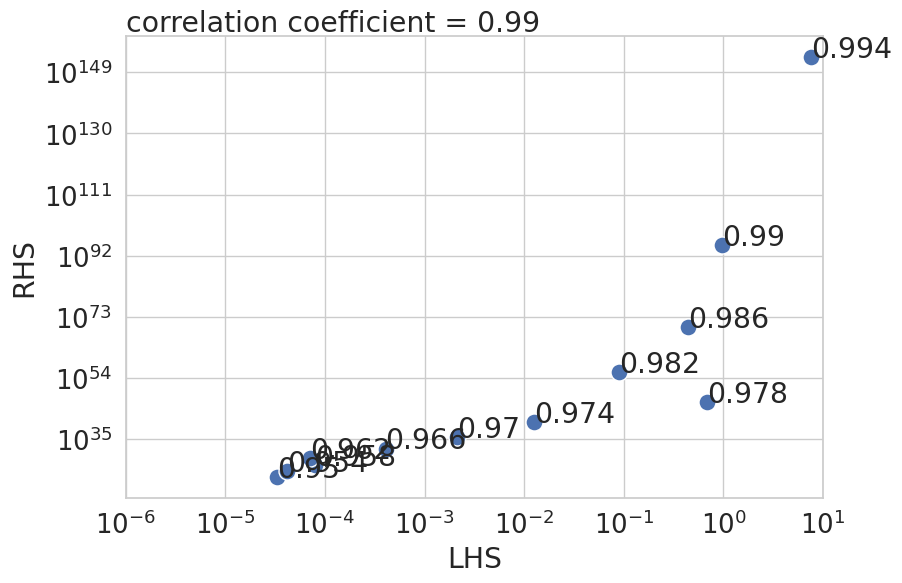}
        \vspace{-1.40em}
        \caption{width=300}
        \label{fig:plot_RHSvsLHS_width300}
    \end{subfigure}
    \vspace{-0.90em}
    \caption{\small Demonstration of the presence of high correlation between the LHS (the true risk) and the RHS (and the derived bound) of equation (\ref{eq:int_burger_bound_th_1}) in Theorem \ref{th:int_burger_gen_error_bound} over PDE setups increasingly close to the singularity. Each experiment is labeled with the value of $\delta$ in the setup of equation \ref{eq:burger_pdes:6} that it corresponds to.}
    \vspace{-1.50em}
    \label{fig:plot_Burgers_loss}
\end{figure}

In \twofigref{fig:plot_RHSvsLHS_width30}{fig:plot_RHSvsLHS_width300} we see that the LHS and the RHS of equation \ref{eq:int_burger_bound_th_1} measured on the trained models is such that the correlation is very high ($\sim 1$)  over multiple values of the proximity parameter -- up to being very close to the blow-up point. We also note that the correlation increases with the width of the neural net, a desirable phenomenon that our bound does capture -- albeit implicitly. 

In Figure \ref{fig:plot_Burgers_varywidth} in the appendix, we illustrate that the upper-bound derived in Theorem \ref{th:int_burger_gen_error_bound} does indeed fall over a reasonable range of widths at a fixed $\delta$.  The mean and the standard deviations plotted therein are obtained over six iterations of the experiment at different random seeds. In Figure \ref{fig:plot_time_burgers} in the appendix, we illustrate the time needed to train the PINN for the $1+1$ Burgers' PDE across a range of $\delta$, representing its proximity to the blow-up. It is evident from the figure that the training time remains approximately constant across all values of $\delta$.

\subsubsection{Testing Against a (2+1)-Dimensional Exact Burgers' Solution with Finite-Time Blow-Up}

From \cite{biazar2009exact} we know that there is an exact finite-time blow-up solution for Burgers' PDE in equation \ref{ndburgers.2} for the case of $d=2$, 
\[ u_1 = \frac{x_1+x_2-2x_1t}{1-2t^2} , ~u_2 = \frac{x_1-x_2-2x_2t}{1-2t^2} \]
where $u_i$ denotes the $i^{\rm th}$ component of the velocity being solved for. Note that at $t=0$, both the above velocities are smooth while they eventually develop singularities at $t=\frac{1}{\sqrt{2}}$ - as is the expected hallmark of non-trivial finite-time blow-up solutions of PDEs. Also note that this singularity is more difficult to solve for since it is blowing up as $\gO(\frac{1}{t^2})$ as compared to the $\gO(\frac{1}{t})$ blow-up in the previous section in one dimension.

We set ourselves to the task of solving for this on a sequence of computational domains $x_1, x_2\in [0,1]$ and $t \in [-\frac{1}{\sqrt{2}} + \delta, \delta]$ where $ \delta \in [0,\frac{1}{\sqrt{2}})$. Hence we have a sequence of PDEs to solve for -- parameterized by $\delta$ and larger $\delta$s getting close to the blow-up. Let ${g}_{x_1,0}(x_2, t)$ and ${g}_{x_1,1}(x_2, t)$ be the boundary conditions for $u_1$ at $x_1 = 0 ~\& ~1$. Let ${g}_{x_2,0}(x_1, t)$ and ${g}_{x_2,1}(x_1, t)$ be the boundary conditions for $u_2$ at $x_2=0 ~\& ~1$ and $u_{1,t_0}$ and $u_{2,t_0}$ with $ t_0 = -\frac{1}{\sqrt{2}} + \delta$ be the initial conditions for the two components of the velocity field. Hence the PDE we seek to solve is,
\begin{equation}
\label{burger2.inviscid}
\small
\left\{\begin{array}{l}
\vu_{t} + (\vu \cdot \nabla) \vu=0
% v_{t}+u v_{x}+v v_{x_2}=0
\end{array}\right.,
\left\{\begin{array}{l}
u_{1,t_0}= \frac{(1+\sqrt{2}-2\delta)x_1 + x_2}{2\delta(\sqrt{2}-\delta)} \\
u_{2,t_0}= \frac{x_1-(1-\sqrt{2}+2\delta)x_2}{2\delta(\sqrt{2}-\delta)}
\end{array}\right.,
\left\{\begin{array}{l}
{g}_{x_1,0}(x_2, t) \coloneqq u_1(x_1=0) =\frac{x_2}{1-2\cdot t^2}\\
{g}_{x_1,1}(x_2, t) \coloneqq u_1(x_1=1)  =\frac{1+x_2-2\cdot t}{1-2\cdot t^2}\\
{g}_{x_2,0}(x_1, t) \coloneqq u_2(x_2=0)   =\frac{x_1}{1-2\cdot t^2}\\
{g}_{x_2,1}(x_1, t) \coloneqq u_2(x_2=1) =\frac{x_1-1-2\cdot t}{1-2\cdot t^2}
\end{array}\right.
\end{equation}
Let $\gN:\R^3 \rightarrow \R^2$ be the neural net to be trained, with output coordinates labeled as $({\gN_{u_1}},{\gN_{u_2}})$. Using this net we define the neural surrogates for solving the above PDE as,
\begin{align*}
    u_{1,\theta} \coloneqq {\gN_{u_1}}(x_1, x_2, t) ~\ u_{2,\theta} \coloneqq {\gN_{u_2}}(x_1,x_2, t)
\end{align*}
Correspondingly we define the PDE population risk, ${\gR}_{pde}$ as, 
\begin{equation}
\label{Burger2.loss1}
\begin{aligned}
{\gR}_{pde} &= \left\| \partial_t \vu_{\theta} + \vu_\theta \cdot \nabla \vu_\theta \right\|_{[0,1]^2 \times[-\frac{1}{\sqrt{2}} + \delta, \delta], \nu_{1}}^{2} \
\end{aligned}
\end{equation}
In the above $\vu_\theta = (u_{1,\theta},u_{2,\theta})$ and $\nu_1$ is a measure on the whose space-time domain $[0,1]^2 \times [-\frac{1}{\sqrt{2}} + \delta, \delta]$. Corresponding to a measure $\nu_2$ on $[0,1] \times [-\frac{1}{\sqrt{2}} + \delta, \delta]$ (first interval being space and the later being time), we define ${\gR}_{s,0}$  and $ {\gR}_{s,1} $ corresponding to violation of the boundary conditions,
\begin{equation}
\label{Burger2.loss3}
\begin{aligned}
{\gR}_{s,0} &= \left\| {u_{1,\theta}} - {g}_{x_1,0}(x_2,t)\right\|_{\{0\} \times[0,1]\times[-\frac{1}{\sqrt{2}} + \delta, \delta], \nu_{2}}^{2} + \left\|{u_{2,\theta}} - {g}_{x_2,0}(x_1,t)\right\|_{[0,1]\times\{0\} \times[-\frac{1}{\sqrt{2}} + \delta, \delta], \nu_{2}}^{2} 
\\
{\gR}_{s,1} &= \left\| {u_{1,\theta}} - {g}_{x_1,1}(x_2,t)\right\|_{\{1\} \times[0,1]\times[-\frac{1}{\sqrt{2}} + \delta, \delta], \nu_{2}}^{2} + \left\|{u_{2,\theta}} - {g}_{x_2,1}(x_1,t)\right\|_{[0,1]\times\{1\} \times[-\frac{1}{\sqrt{2}} + \delta, \delta], \nu_{2}}^{2}
\end{aligned}
\end{equation}
For a choice of measure $\nu_3$ on the spatial volume $[0,1]^2$ we define ${\gR}_{t} $ corresponding to the violation of initial conditions $\vu_{t_0} = (u_1 (t_0), u_2(t_0))$,
\begin{equation}
\label{Burger2.loss4}
{\gR} _{t} = \left\| \vu_\theta - \vu_{t_0} \right\|_{[0,1]^2 ,t=t_0, \nu_{3}}^{2} \
% {\gR} _{5} = \left\| {u_1} - u_{1,t_0}\right\|_{[0,1]^2 ,t=t_0, \nu_{3}}^{2},\
% {\gR} _{6} = \left\|{u_2} - u_{2,t_0}\right\|_{[0,1]^2 ,t=t_0, \nu_{3}}^{2}
\end{equation}
Thus the population risk we are looking to minimize is, $
{\gR}_{2+1}  \coloneqq  {\gR} _{pde} + {\gR}_{s,0}+{\gR}_{s,1} + {\gR}_{t}
$
% Experiments reveal that a successful implementation is possible by choosing uniformly at random $800$ points in the space-time bulk, $800$ points at $t=t_0$, and $200$ points at each of the surfaces at $x_1=0,1$ and $x_2=0,1$ and training for $3\times 10^{4}$ epochs at a learning rate of $10^{-3}$ learning-rate. 

We note that for the exact solution given above the constants in Theorem \ref{th:ndburgers.error_upperbound} evaluate to, 
\begin{align*}
    C_1 &= 2^2 \norm{\nabla \vu_\theta}_{L^\infty(\Omega)}+ 1 + 2^2 \max_{t=-\frac{1}{\sqrt{2}} + \delta, \delta}\left\{\abs{\frac{1-2t}{1-2t^2}} + \abs{\frac{1}{1-2t^2}}, \abs{\frac{1}{1-2t^2}} + \abs{\frac{1+2t}{1-2t^2}} \right\} \\
    C_2 &= \int_D \norm{\gR_t}_2^2 \dd{\vx} + \int_{\Tilde{\Omega}} \norm{\gR_{pde}}_2^2  \dd{\vx} \dd{t} + 2^2 \norm{\nabla \vu_\theta}_{L^\infty(\Omega)} \int_{\Tilde{\Omega}} \norm{\vu_\theta}_2^2 \dd{\vx} \dd{t} \\
    &+ 2^2 \max_{t=-\frac{1}{\sqrt{2}} + \delta, \delta}\left\{\abs{\frac{1-2t}{1-2t^2}} + \abs{\frac{1}{1-2t^2}}, \abs{\frac{1}{1-2t^2}} + \abs{\frac{1+2t}{1-2t^2}} \right\} \int_{\Tilde{\Omega}} \norm{\vu}_2^2 \dd{\vx} \dd{t}  \\
     &= \int_D \norm{\gR_t}_2^2 \dd{\vx} + \int_{\Tilde{\Omega}} \norm{\gR_{pde}}_2^2  \dd{\vx} \dd{t} + 2^2 \norm{\nabla \vu_\theta}_{L^\infty(\Omega)} \int_{\Tilde{\Omega}} \norm{\vu_\theta}_2^2 \dd{\vx} \dd{t} \\
    &+ 2^2 \max_{t=-\frac{1}{\sqrt{2}} + \delta, \delta}\left\{\abs{\frac{1-2t}{1-2t^2}} + \abs{\frac{1}{1-2t^2}}, \abs{\frac{1}{1-2t^2}} + \abs{\frac{1+2t}{1-2t^2}} \right\} \left[\left.\left[\frac{11t - 7}{12(1-2t^2)} + \frac{5t + 1}{12 (1-2t^2)}\right]\right|_{t=\delta}\right. \\
    &- \left.\left.\left[\frac{11t - 7}{12(1-2t^2)} + \frac{5t + 1}{12 (1-2t^2)}\right]\right|_{t=-\frac{1}{\sqrt{2}}+\delta}\right]
    % _{t=-\frac{1}{\sqrt{2}}+\delta}^{\delta}
\end{align*}

\begin{figure}[!h]
    \centering
    \begin{subfigure}[h]{0.49\textwidth}
        \includegraphics[width=\textwidth]{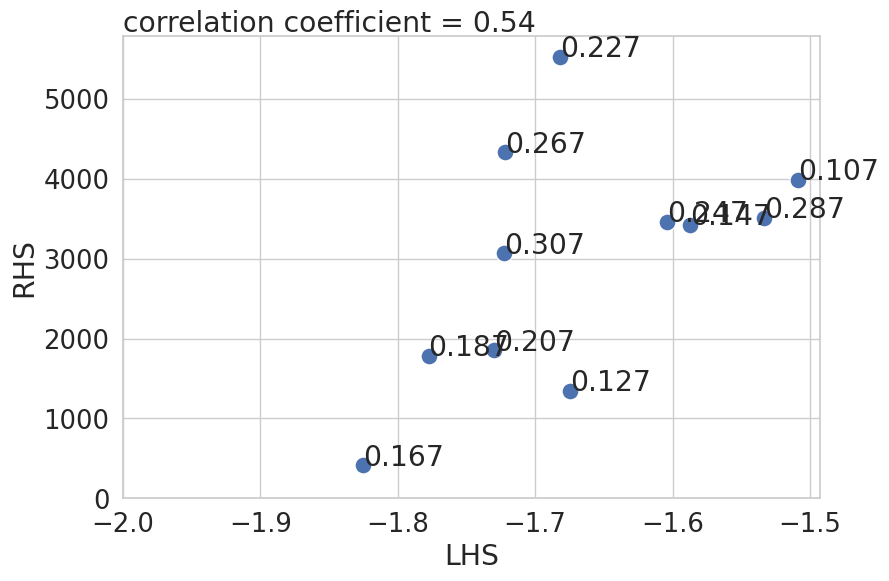}
        \caption{width=30}
        \label{fig:plot_RHSvsLHS_width30_first11}
    \end{subfigure}
    \hfill
    \begin{subfigure}[h]{0.49\textwidth}
        \includegraphics[width=\textwidth]{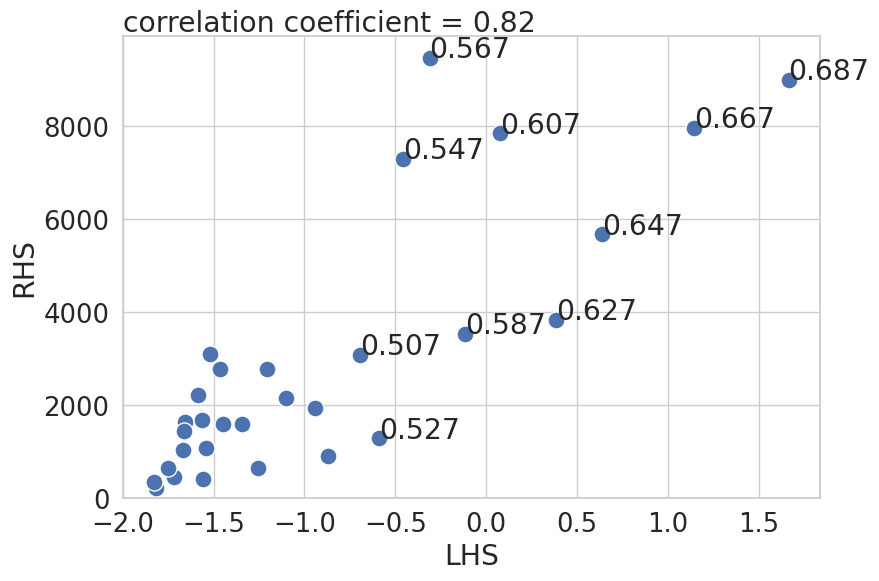}
        \caption{width=100}
        \label{fig:plot_RHSvsLHS_width100}
    \end{subfigure}
    \vspace{-0.90em}
    \caption{\small These plots show the behaviour of LHS (the true risk) and RHS (the derived bound) of equation (\ref{ndburgers.upperbound}) in Theorem \ref{th:ndburgers.error_upperbound} for different values of the $\delta$ parameter that quantifies proximity to the blow-up point. In the left plot each point is marked with the value of the $\delta$ at which the experiment is done and the right figure, for clarity, this is marked only for experiments at $\delta > \frac{1}{2}$.}
    \vspace{-1.50em}
    \label{fig:plot_ndburgers}
\end{figure}

In Figure \ref{fig:plot_ndburgers} we see the true risk and the derived bound in Theorem \ref{th:ndburgers.error_upperbound} for depth $6$ neural nets obtained by training on the above loss. The experiments show that the insight from the previous demonstration continues to hold and more vividly so. Here, for the experiments at low width ($30$) the correlation stays around $0.50$ and only until $\delta = 0.307$, and beyond that it decreases rapidly. However, for experiments at width $100$ the correlation remains close to $0.80$ for $\delta$ much closer to the blow-up at $t = \frac{1}{\sqrt{2}}$.
% The experiments show that the correlation between the RHS and LHS increases from $0.41$ to $0.82$ when the net size is increased from $30$ to $100$.

%\clearpage 
\section{Conclusion}

In this work we have taken some of the first-of-its kind steps to initiate research into understanding the ability of neural nets to solve PDEs at the edge of finite-time blow-up. Our work suggests a number of exciting directions of future research. Firstly, more sophisticated modifications to the PINN formalism could be found to solve PDEs specifically near finite-time blow-ups. 

Secondly, we note that it remains an open question to establish if there is any PINN risk for the $d+1$-dimensional Burgers, for $d>1$, that is stable by the condition stated in \cite{wang2022stability}, as was shown to be true in our $1+1$-dimensional Burgers' in Theorem~\ref{th:int_burger_gen_error_bound}.

In \cite{luo-hou_pnas} the authors had given numerical studies to suggest that 3D incompressible Euler PDEs can develop finite-time singularities from smooth initial conditions for the fluid velocity. For their setup of axisymmetric fluid flow they conjectured a simplified model for the resultant flow near the outer boundary of the cylinder. Self-similar finite-time blow-ups for this model's solutions were rigorously established in \cite{chen2022asymptotically} - and it was shown that an estimate of its blowup-exponent is very close to the measured values of the 3D Euler PDE. 

In the seminal paper \cite{elgindi2021euler} it was shown that the unique local solution to 3D incompressible Euler PDEs can develop finite-time singularities despite starting from a divergence-free and odd initial velocity in $\mathbb{C}^{1,\alpha}$ and initial vorticity bounded as $\sim \frac{1}{1+\norm{x}^\alpha}$. This breakthrough was built upon to prove the existence of finite time singularity in 2D Boussinesq PDE in \cite{chen2021finite}.  

In \cite{Luo-Hou} it was highlighted that there is an association between blow-ups in 3D Euler and 2D Boussinesq PDEs. In \cite{wang2022asymptotic}, the authors investigated the ability for PINNs to detect the occurrence of self-similar blow-ups in 2D Boussinesq PDE. A critical feature of this experiment was its use of the unconventional regularizer on the gradients of the neural surrogate with respect to its inputs. In light of this, we posit that a very interesting direction of research would be to investigate if a theoretical analysis of the risk bound for such losses can be used as a method of detection of the blow-up. 

% \section*{Acknowledgements}
% We would like to thank Siddhartha Mishra, Tim De Ryck and Alejandro Frangi for the insightful discussions during this work. We thank Pratibhamoy Das for his valuable feedback on the draft at various stages of this work. 

\clearpage 
% \printbibliography
\section*{References}
% \title[]{}
\bibliographystyle{iopart-num}
\bibliography{references}

\clearpage 
\renewcommand{\thetheorem}{\Alph{section}\arabic{theorem}}
\renewcommand{\thelemma}{\Alph{section}\arabic{lemma}}
\title[Appendix]{}
\appendix
\section{Proofs for the Main Theorems}
\subsection{Proof of Theorem \ref{th:ndburgers.error_upperbound}}
\label{subsec:ndburgers.proof}
\begin{proof}
Let $\vu$ be the actual solution of the (d+1)-dimensional Burgers' PDE and $\vu_\theta$ the predicted solution by the PINN with parameters $\theta$.
Let's define
\begin{align*}
    f(\vu) \coloneqq \frac{\norm{\vu}_2^2}{2} \\
    \hat \vu \coloneqq \vu_\theta - \vu
\end{align*}
Then we can write the (d+1)-dimenional Burgers' and it's residual as
\begin{align}
\partial_t \vu + (\vu \cdot \nabla) \vu = 0 \label{2dburgers.1}\\
\gR_{\rm{pde}} \coloneqq \partial_t \vu_\theta + (\vu_\theta \cdot \nabla) \vu_\theta \label{2dburgers.rpde}
\end{align}
Now multiplying \plaineqref{2dburgers.rpde} with $\vu_\theta$ on both sides we get
\begin{align}
    \vu_\theta \cdot \gR_{\rm{pde}} = \partial_t f(\vu_\theta) + \vu_\theta \cdot \nabla f(\vu_\theta) \label{2dburgers.urpde}
\end{align}
Similarly, multiplying both sides of \plaineqref{2dburgers.1} with $\vu$ we get
\begin{align}
    \partial_t f(\vu) + \vu \cdot \nabla f(\vu) = 0 \label{2dburgers.u}
\end{align}
Some calculation on \plaineqref{2dburgers.1} and \plaineqref{2dburgers.rpde} yields
\begin{align}
    &\vu \cdot \left[ (\partial_t \vu_\theta + (\vu_\theta \cdot \nabla) \vu_\theta - \gR_{\rm{pde}}) - (\partial_t \vu + (\vu \cdot \nabla) \vu ) \right] = - \hat\vu \cdot \left[ \partial_t \vu + (\vu \cdot \nabla) \vu \right] \nonumber\\
    \implies &\partial_t (\vu \cdot \hat\vu) + \vu \cdot ((\vu_\theta\cdot\nabla)\vu_\theta) - \vu \cdot ((\vu\cdot\nabla)\vu) - \vu \cdot \gR_{\rm{pde}} = - \hat\vu \cdot ((\vu \cdot \nabla) \vu) \nonumber\\
    \implies &\partial_t (\vu \cdot \hat\vu) + \vu \cdot ((\vu_\theta\cdot\nabla)\vu_\theta) - \vu \cdot \nabla f(\vu) - \vu \cdot \gR_{\rm{pde}} = - \hat\vu \cdot ((\vu \cdot \nabla) \vu) \label{2dburgers.2}
\end{align}
Let
\begin{align}
    S &\coloneqq \frac{1}{2} \hat\vu \cdot \hat\vu \\
    \implies \partial_t S &= \partial_t f(\vu_\theta) - \partial_t f(\vu) - \partial_t(\vu \cdot \hat\vu) \nonumber\\
    &= \left[\vu_\theta \cdot \gR_{\rm{pde}} - \vu_\theta \cdot \nabla f(\vu_\theta)\right] + \left[\vu \cdot \nabla f(\vu)\right] \nonumber\\
    &- \left[- \vu \cdot ((\vu_\theta\cdot\nabla)\vu_\theta) + \vu \cdot \nabla f(\vu) + \vu \cdot \gR_{\rm{pde}} - \hat\vu \cdot ((\vu \cdot \nabla) \vu) \right] \nonumber\\
    &= \hat\vu \cdot \gR_{\rm{pde}} - \vu_\theta \cdot \nabla f(\vu_\theta) + \vu \cdot ((\vu_\theta\cdot\nabla)\vu_\theta) + \hat\vu \cdot ((\vu \cdot \nabla) \vu) \nonumber\\
    &= \hat\vu \cdot \gR_{\rm{pde}} + \hat\vu \cdot ((\vu \cdot \nabla) \vu - (\vu_\theta \cdot \nabla) \vu_\theta) \nonumber
    % \Tilde{H} &\coloneqq \vu_\theta \cdot (\vu_\theta \cdot \nabla \vu_\theta) - \vu \cdot (\vu \cdot \nabla \vu) - \nabla \cdot \left( \vu \left[ f(\vu_\theta) - f(\vu)\right] \right) \\
    % T_1 &\coloneqq \left[ f(\vu_\theta) - f(\vu)\right] \nabla \cdot \vu - \hat\vu \cdot ((\vu \cdot \nabla) \vu)
\end{align}
here we represent the spatial domain $[0,1]\times[0,1]$ by $D$ and use $\Omega$ to represent the $D \times [-\frac{1}{\sqrt{2}}+\delta,\delta]$. We then define
\begin{align}
    \gT \coloneqq \hat\vu \cdot ((\vu \cdot \nabla) \vu) \\
    \Tilde{H} \coloneqq \hat\vu \cdot ((\vu_\theta \cdot \nabla) \vu_\theta) 
\end{align}
And this leads to,
\begin{align} 
    \partial_t S + \Tilde{H} = \hat\vu \cdot \gR_{\rm{pde}} + \gT
\end{align}
Thus we have the inequalities,
\begin{align}
    \int_D \partial_t \norm{\hat \vu}_2^2 \dd{\vx} &\leq \int_D \norm{\hat \vu}_2^2  \dd{\vx} + \int_D \norm{\gR_{pde}}_2^2  \dd{\vx} \nonumber\\
    &+ 2 \int_D \hat \vu \cdot ((\vu \cdot \nabla)\vu) \dd{\vx} - 2 \int_D \hat \vu \cdot ((\vu_\theta \cdot \nabla) \vu_\theta) \dd{\vx} \nonumber\\ %1
    &[\mbox{using Lemma}~\ref{lemma:NS.upperbound2}] \nonumber\\
    &\leq \int_D \norm{\hat \vu}_2^2  \dd{\vx} + \int_D \norm{\gR_{pde}}_2^2  \dd{\vx} \nonumber\\
    &+ 2 d^2 \norm{\nabla \vu}_{L^\infty(\Omega)} \int_D \norm{\vu}_2 \norm{\hat \vu}_2 \dd{\vx} + 2 d^2 \norm{\nabla \vu_\theta}_{L^\infty(\Omega)} \int_D \norm{\vu_\theta}_2 \norm{\hat \vu}_2 \dd{\vx} \nonumber\\ %2
   &\leq \int_D \norm{\hat \vu}_2^2  \dd{\vx} + \int_D \norm{\gR_{pde}}_2^2  \dd{\vx} \nonumber\\
    &+ d^2 \norm{\nabla \vu}_{L^\infty(\Omega)} \int_D \left[\norm{\vu}_2^2 + \norm{\hat \vu}_2^2 \right] \dd{\vx} + d^2 \norm{\nabla \vu_\theta}_{L^\infty(\Omega)} \int_D \left[ \norm{\vu_\theta}_2^2 + \norm{\hat \vu}_2^2 \right] \dd{\vx} \nonumber\\ %3
    &\leq \left[ 1 + d^2 \norm{\nabla \vu}_{L^\infty(\Omega)} + d^2 \norm{\nabla \vu_\theta}_{L^\infty(\Omega)}\right] \int_D \norm{\hat \vu}_2^2  \dd{\vx} \nonumber\\
    &+ \int_D \norm{\gR_{pde}}_2^2  \dd{\vx} + d^2 \norm{\nabla \vu}_{L^\infty(\Omega)} \int_D \norm{\vu}_2^2 \dd{\vx} + d^2 \norm{\nabla \vu_\theta}_{L^\infty(\Omega)} \int_D \norm{\vu_\theta}_2^2 \dd{\vx} \nonumber\\
   &\leq C_1 \int_D \norm{\hat \vu}_2^2  \dd{\vx} + \int_D \norm{\gR_{pde}}_2^2  \dd{\vx} \nonumber\\
    &+ d^2 \norm{\nabla \vu}_{L^\infty(\Omega)} \int_D \norm{\vu}_2^2 \dd{\vx} + d^2 \norm{\nabla \vu_\theta}_{L^\infty(\Omega)} \int_D \norm{\vu_\theta}_2^2 \dd{\vx} \label{Ns.upperbound.integral1}
\end{align}
where
\begin{align*}
C_1 &= d^2 \norm{\nabla \vu_\theta}_{L^\infty(\Omega)}\\
&+ 1 + d^2 \norm{\nabla \vu}_{L^\infty(\Omega)}
\end{align*}
Lets define the domain $D\times[-\frac{1}{\sqrt{2}} + \delta, \Tilde{\delta}]$ by $\Tilde{\Omega}$ where $\Tilde{\delta} \in [-\frac{1}{\sqrt{2}} + \delta, \delta)$. 

Integrating over $\Tilde{\Omega}$ we get
\begin{align}
    \nonumber \int_{\Tilde{\Omega}} \partial_t \norm{\hat \vu}_2^2 \dd{\vx} \dd{t} &\leq C_1 \int_{\Tilde{\Omega}} \norm{\hat \vu}_2^2  \dd{\vx} \dd{t} + \int_{\Tilde{\Omega}} \norm{\gR_{pde}}_2^2  \dd{\vx} \dd{t}\\
    \nonumber &+ d^2 \norm{\nabla \vu}_{L^\infty(\Omega)} \int_{\Tilde{\Omega}} \norm{\vu}_2^2 \dd{\vx} \dd{t} + d^2 \norm{\nabla \vu_\theta}_{L^\infty(\Omega)} \int_{\Tilde{\Omega}} \norm{\vu_\theta}_2^2 \dd{\vx} \dd{t} \\ %1
    %\implies \int_D \norm{\hat \vu(\vx, \Tilde{\delta})}_2^2 \dd{\vx} 
    \nonumber &\leq \int_D \norm{\gR_t}_2^2 \dd{\vx} + C_1 \int_{\Tilde{\Omega}} \norm{\hat \vu}_2^2  \dd{\vx} \dd{t} + \int_{\Tilde{\Omega}} \norm{\gR_{pde}}_2^2  \dd{\vx} \dd{t} \nonumber\\
    &+ d^2 \norm{\nabla \vu}_{L^\infty(\Omega)} \int_{\Tilde{\Omega}} \norm{\vu}_2^2 \dd{\vx} \dd{t} + d^2 \norm{\nabla \vu_\theta}_{L^\infty(\Omega)} \int_{\Tilde{\Omega}} \norm{\vu_\theta}_2^2 \dd{\vx} \dd{t} \nonumber\\ %2
    &\leq \int_D \norm{\gR_t}_2^2 \dd{\vx} + C_1 \int_{\Omega} \norm{\hat \vu}_2^2  \dd{\vx} \dd{t} + \int_{\Omega} \norm{\gR_{pde}}_2^2  \dd{\vx} \dd{t} \nonumber\\
    &+ d^2 \norm{\nabla \vu}_{L^\infty(\Omega)} \int_{\Omega} \norm{\vu}_2^2 \dd{\vx} \dd{t} + d^2 \norm{\nabla \vu_\theta}_{L^\infty(\Omega)} \int_{\Omega} \norm{\vu_\theta}_2^2 \dd{\vx} \dd{t} \nonumber\\ %3
     &\leq C_1 \int_{\Tilde{\Omega}} \norm{\hat \vu}_2^2  \dd{\vx} \dd{t} + C_2 \label{Ns.upperbound.integral2}
\end{align}
% \note{change $C_2$ to have $\Omega$ instead of $\Tilde{\Omega}$}
where,
\begin{align*}
C_2 &= \int_D \norm{\gR_t}_2^2 \dd{\vx} + \int_{\Tilde{\Omega}} \norm{\gR_{pde}}_2^2  \dd{\vx} \dd{t} + d^2 \norm{\nabla \vu_\theta}_{L^\infty(\Omega)} \int_{\Tilde{\Omega}} \norm{\vu_\theta}_2^2 \dd{\vx} \dd{t} \\
&+ d^2 \norm{\nabla \vu}_{L^\infty(\Omega)} \int_{\Tilde{\Omega}} \norm{\vu}_2^2 \dd{\vx} \dd{t}
\end{align*}

Applying Gronwall's inequality on \plaineqref{Ns.upperbound.integral2} we get
\begin{align}
    \int_D \norm{\hat \vu(\vx, \Tilde{\delta})}_2^2 \dd{\vx} &\leq C_2 + \int_{-\frac{1}{\sqrt{2}} + \delta}^{\Tilde{\delta}} C_2 C_1 e^{\int_t^\delta C_1 \dd{s}} \dd{t} %\nonumber\\
    %\implies \int_D \norm{\hat \vu(\vx, \Tilde{\delta})}_2^2 \dd{\vx} &
    \leq C_2 \left[ 1 + \int_{-\frac{1}{\sqrt{2}} + \delta}^{\Tilde{\delta}} C_1 e^{\frac{C_1}{\sqrt{2}}} \dd{t} \right] \label{Ns.upperbound.integral3}
\end{align}

Integrating \plaineqref{Ns.upperbound.integral3} over $\dd{\Tilde{\delta}}$ we get
\begin{align}
    \nonumber \int_\Omega \norm{\hat \vu(\vx, \Tilde{\delta})}_2^2 \dd{\vx}\dd{\Tilde{\delta}} &\leq C_2 \int_{\frac{-1}{\sqrt{2}} + \delta}^{\delta}\left[ 1 + \int_{-\frac{1}{\sqrt{2}} + \delta}^{\Tilde{\delta}} C_1 e^{\frac{C_1}{\sqrt{2}}} \dd{t} \right] \dd{\Tilde{\delta}}\nonumber\\
    &\leq C_2 \left[ \frac{-1}{\sqrt{2}} + \int_{\frac{-1}{\sqrt{2}} + \delta}^{\delta}\int_{-\frac{1}{\sqrt{2}} + \delta}^{\Tilde{\delta}} C_1 e^{\frac{C_1}{\sqrt{2}}} \dd{t}\dd{\Tilde{\delta}} \right] \nonumber\\
     &\leq C_2 \left[ \frac{-1}{\sqrt{2}} +  \frac{C_1}{4} e^{\frac{C_1}{\sqrt{2}}} \right] \nonumber\\
     \implies\log\left( \int_\Omega \norm{\hat \vu(\vx, \Tilde{\delta})}_2^2 \dd{\vx}\dd{\Tilde{\delta}}\right) &\leq \log{\left(\frac{C_1 C_2}{4}\right)} + \frac{C_1}{\sqrt{2}} \label{Ns.upperbound.integral4}
\end{align}
\end{proof}

%\subsection{A $2+1$ Burgers' Solution with Finite-Time Blow-Up}

% Let $u,v : \R^3 \rightarrow \R$ be the two components of the velocity to be solved for Burgers' PDE in two spatial dimensions. On the computational domain $x, y\in [0,1]$ and $t \geq 0$ we define ${g}_{x,0}(y, t)$ and ${g}_{x,1}(y, t)$ as the boundary conditions for $u$ at $x=0, 1$, ${g}_{y,0}(x, t)$ and ${g}_{y,1}(x, t)$ as the boundary conditions for $v$ at $y=0,1$ and $u_{0}$ and $v_{0}$ as the initial conditions for the two velocity fields. Then from \cite{biazar2009exact} we know that there is an exact finite-time blow-up solution of the following setting for Burgers' PDE, 

\subsection{Useful Lemmas}
\begin{lemma}\label{lemma:NS.upperbound2}
    \begin{align}
        \int_D p \cdot ((q \cdot \nabla)r) dx &= \int_D \left[ \sum_{i=1}^d p_i \left( q \cdot  \nabla r_i \right) \right] dx\nonumber\\
        &\leq \int_D \left[ \sum_{i=1}^d \norm{p}_2 \left( \norm{q}_2 \norm{\nabla r_i}_2 \right) \right] dx \nonumber\\
        &\leq \int_D \norm{p}_2\norm{q}_2 \left[ \sum_{i=1}^d \norm{\nabla r_i}_2 \right] dx \nonumber\\
        &\leq \int_D \norm{p}_2\norm{q}_2 \left[ \sum_{i=1}^d d\norm{\nabla r}_{L^\infty(\Omega)} \right] dx \nonumber\\
        &\leq d^2\norm{\nabla r}_{L^\infty(\Omega)} \int_D \norm{p}_2\norm{q}_2 dx
    \end{align}
\end{lemma}
\clearpage
%\section{Proofs of The Main Theorems} 

% \note{proofs should be in the same order as the theorems!}

\subsection{Proof of Theorem \ref{th:int_burger_gen_error_bound}} \label{proof:int_burger_gen_error_bound}
\begin{proof}
We define
\begin{align*}
    f(u) = \frac{u^2}{2}
\end{align*}
which means the first equation in (\ref{eq:burger_pdes:6}) can be written as
\begin{align}
    u_t + f(u)_x = 0
    \label{eq:int_burger_proof_0}
\end{align}
Then we define the entropy flux function as
\begin{align*}
    \gQ(u) = \int_a^u sf'(s) ds \quad \mbox{for any }a \in \R
\end{align*}
Let $\hat{u} = u^* - u$. From (\ref{eq:burger_pdes_residual:1}) we get
\begin{align}
    \partial_t \left( \frac{(u^*)^2}{2} \right) + \partial_x \gQ(u^*) = u^* \gR_{int,\theta*}
    \label{eq:int_burger_proof_1}
\end{align}
and from (\ref{eq:int_burger_proof_0}) we obtain
\begin{align}
    \partial_t \left( \frac{u^2}{2} \right) + \partial_x \gQ(u) = 0
    \label{eq:int_burger_proof_2}
\end{align}
Some calculation on (\ref{eq:int_burger_proof_0}) and (\ref{eq:burger_pdes_residual:1}) yields
\begin{align}
    \partial_t (u\hat u) + \partial_x (u(f(u^*) - f(u))) =  [ f(u^*) - f(u) - \hat u f'(u)] u_x + u\gR_{int,\theta^*}
    \label{eq:int_burger_proof_3}
\end{align}
Subtracting (\ref{eq:int_burger_proof_3}) and (\ref{eq:int_burger_proof_2}) from (\ref{eq:int_burger_proof_1}) we get
\begin{align}
    \partial_t S(u,u^*) + \partial_x H(u,u^*) = \hat u\gR_{int,\theta^*} + T_1
    \label{eq:int_burger_proof_4}
\end{align}
with,
\begin{align*}
    S(u,u^*) &\coloneqq \frac{(u^*)^2}{2} - \frac{u^2}{2} - \hat u u = \frac{1}{2} \hat u^2, \\
    H(u,u^*) &\coloneqq \gQ(u^*) - \gQ(u) - u(f(u^*) - f(u)), \\
    T_1 &= - \left[ f(u^*) - f(u) -f'(u) \hat u \right] u_x
\end{align*}
As flux f is smooth, we can apply Taylor expansion\footnote{with the Lagrange form of the remainder} on $T_1$ and expand $ f(u^*)$ at $u$,
\begin{align*}
    T_1 &= -\left[ \cancel{f(u)} + \cancel{f'(u)\hat u} + \frac{f''(u+\gamma(u^*-u))}{2}(u^* - u)^2 - \cancel{f(u)} - \cancel{f'(u)\hat u} \right] u_x \\
    &\left[ \mbox{ where } \gamma \in (0,1) \right] \\
    &= - \frac{1}{2}\ \cancelto{1}{f''(u+\gamma(u^*-u))}\ \hat u^2 u_x \\
    &= - \frac{1}{2} \hat u^2 u_x
\end{align*}
% Now, for some value of $\gamma \in (0,1),\   T_1$ will be
% \begin{align*}
%     T_1 = -\frac{1}{2} f''(u + \gamma(u^* - u)) \hat u^2 u_x \nonumber\\
% \end{align*}
Hence it can be reasonably bounded by
\begin{align}
    |T_1| \leq \norm{u_x}_{L^\infty} \hat u^2
\end{align}
where $C_{u_x}$ is given by $C_{u_x} = \norm{u_x}_{L^\infty}$. Next, we integrate (\ref{eq:int_burger_proof_4}) over the domain (-1,1)
\begin{align}
    \int_{-1}^1 \partial_t S(u,u^*) dx &= - \int_{-1}^1 \partial_x H(u,u^*) dx + \int_{-1}^1 \hat u\gR_{int,\theta^*} dx + \int_{-1}^1 T_1 dx \nonumber \\
    \Rightarrow\ \dv{}{t} \int_{-1}^1 \frac{\hat u^2(x,t)}{2} dx &\leq H(u(-1,t),u^*(-1,t)) - H(u(1,t),u^*(1,t)) \nonumber\\
    &+ C_{u_x} \int_{-1}^1 \hat u^2(x,t) dx + \int_{-1}^1 \hat u(x,t) \gR_{int,\theta^*}(x,t) dx \nonumber\\
    \Rightarrow\ \dv{}{t}\int_{-1}^1\hat u^2(x,t) dx &\leq 2H(u(-1,t),u^*(-1,t)) - 2H(u(1,t),u^*(1,t)) \nonumber\\
    &+ 2C_{u_x} \int_{-1}^1 \hat u^2(x,t) dx + \int_{-1}^1 (\gR_{int,\theta^*}^2(x,t) + \hat u^2(x,t)) dx \nonumber\\
    \Rightarrow\ \dv{}{t}\int_{-1}^1\hat u^2(x,t) dx &\leq 2H(u(-1,t),u^*(-1,t)) - 2H(u(1,t),u^*(1,t)) \nonumber\\
    &+ C \int_{-1}^1 \hat u^2(x,t) dx + \int_{-1}^1 \gR_{int,\theta^*}^2(x,t) dx
    \label{eq:int_burger_proof_5}
\end{align}

where $C = 1 + 2C_{u_x}$. We can estimate $H(u(1,t),u^*(1,t))$ using (\ref{eq:burger_pdes:6})

\begin{align*}
    H(u(1,t),u^*(1,t)) &=\gQ(u^*(1,t)) - \gQ(u(1,t)) - u(1,t)(f(u^*(1,t)) - f(u(1,t))) \\
    &= \gQ'(\gamma_1\gR_{sb,1,\theta^*}(t))\gR_{sb,1,\theta^*}(t) - \frac{u(1,t)}{2}\left[ \hat u(1,t) \left[u^*(1,t) + u(1,t)\right] \right] \\
    &[\mbox{ for some } \gamma_1 \in (0,1) \mbox{ by the mean-value theorem}] \nonumber\\
    &= \gamma_1 f'(\gamma_1 \gR_{sb,1,\theta^*})\gR_{sb,1,\theta^*}^2(t) - \frac{u(1,t)}{2}\left[ \gR_{sb,1,\theta^*} + 2u(1,t) \right]\gR_{sb,1,\theta^*} \\
    &= \gamma_1 \left[f'(\gamma_1 \gR_{sb,1,\theta^*}) - \frac{u(1,t)}{2}\right]\gR_{sb,1,\theta^*}^2(t) - u^2(1,t)\gR_{sb,1,\theta^*} \\
    &\leq C_{2b} \gR_{sb,1,\theta^*}^2(t) + u^2(1,t)|\gR_{sb,1,\theta^*}| \nonumber\\
    &\mbox{ with } C_{2b} = C_{2b} (\norm{f'}_\infty, \norm{u}_{C^1([-1,1]\times[-1+\delta,\delta])})
\end{align*}

Similarly we can estimate
\begin{align*}
    H(u(-1,t),u^*(-1,t)) \leq C_{2b} \gR_{sb,-1,\theta^*}^2(t) + u^2(-1,t)|\gR_{sb,-1,\theta^*}|
\end{align*}

% \clearpage 
Now, we can integrate (\ref{eq:int_burger_proof_5}) over the time interval $[-1+\delta, \bar \delta]$ for any $\bar \delta \in [-1+\delta, \delta]$ and use the above inequalities along with (\ref{eq:burger_pdes_residual:2})

\begin{align}
    \int_{-1+\delta}^{\bar\delta}\dv{}{t}\int_{-1}^1\hat u^2(x,t) dx dt &\leq \int_{-1+\delta}^{\bar\delta} (2H(u(-1,t),u^*(-1,t)) - 2H(u(1,t),u^*(1,t))) dt \nonumber\\
    &+ \int_{-1+\delta}^{\bar\delta} C \int_{-1}^1 \hat u^2(x,t) dx dt + \int_{-1+\delta}^{\bar\delta} \int_{-1}^1 \gR_{int,\theta^*}^2(x,t) dx dt \nonumber
\end{align}
\begin{align}
    \Rightarrow\ \int_{-1}^1\hat u^2(x,\bar\delta) dx &\leq \int_{-1}^1\hat u^2(x,-1+\delta) dx + 2 C_{2b} \left[\int_{-1+\delta}^{\delta}\gR_{sb,-1,\theta^*}^2 (t) dt + \int_{-1+\delta}^{\delta}\gR_{sb,1,\theta^*}^2 (t) dt \right] \nonumber\\
    &+ 2 \left[\int_{-1+\delta}^{\delta} u^2(-1,t) |\gR_{sb,-1,\theta^*}| dt + \int_{-1+\delta}^{\delta} u^2(1,t) |\gR_{sb,1,\theta^*}| dt \right] \nonumber\\
    &+ C \int_{-1+\delta}^{\bar \delta}\int_{-1}^1 \hat u^2(x,t) dx dt + \int_{-1+\delta}^{\delta} \int_{-1}^1 \gR_{int,\theta^*}^2(x,t) dx dt \nonumber\\
    \nonumber\\
    \Rightarrow\ \int_{-1}^1\hat u^2(x,\bar\delta) dx &\leq \int_{-1}^1 \gR_{tb,\theta^*}(x) dx + 2 C_{2b} \left[\int_{-1+\delta}^{\delta}\gR_{sb,-1,\theta^*}^2 (t) dt + \int_{-1+\delta}^{\delta}\gR_{sb,1,\theta^*}^2 (t) dt \right] \nonumber\\
    &+ 2 C_{1b} \left[\int_{-1+\delta}^{\delta} |\gR_{sb,-1,\theta^*}| dt + \int_{-1+\delta}^{\delta} |\gR_{sb,1,\theta^*}| dt \right] \nonumber\\
    &+ C \int_{-1+\delta}^{\bar \delta}\int_{-1}^1 \hat u^2(x,t) dx dt + \int_{-1+\delta}^{\delta} \int_{-1}^1 \gR_{int,\theta^*}^2(x,t) dx dt \nonumber\\
    % \displaybreak
    \nonumber &\text{ where } C_{1b} = C_{1b}(\norm{u(1,t)}_{L^\infty})
    \\
    &\leq \int_{-1}^1 \gR_{tb,\theta^*}(x) dx + 2 \bar C_{2b} \left[\int_{-1+\delta}^{\delta}\gR_{sb,-1,\theta^*}^2 (t) dt + \int_{-1+\delta}^{\delta}\gR_{sb,1,\theta^*}^2 (t) dt \right] \nonumber\\
    &+ 2 C_{1b} (\delta - (\delta - 1) )^\frac{1}{2} \left[ \left(\int_{-1+\delta}^{\delta} \gR^2_{sb,-1,\theta^*} dt\right)^\frac{1}{2} + \left(\int_{-1+\delta}^{\delta} \gR^2_{sb,1,\theta^*} dt\right)^\frac{1}{2} \right] \nonumber\\
    &+ C \int_{-1+\delta}^{\bar\delta}\int_{-1}^1 \hat u^2(x,t) dx dt + \int_{-1+\delta}^{\delta} \int_{-1}^1 \gR_{int,\theta^*}^2(x,t) dx dt \nonumber\\
    &\nonumber \text{by using Holder's inequality}\\
    &\leq C_T + C \int_{-1+\delta}^{\bar\delta}\int_{-1}^1 \hat u^2(x,t) dx dt \nonumber
    \end{align}
    \begin{align}
    \mbox{where}\ C_T = \int_{-1}^1 \gR_{tb,\theta^*}(x) \dd{x} + 2 C_{2b} \left[\int_{-1+\delta}^{\delta}\gR_{sb,-1,\theta^*}^2 (t) \dd{t} + \int_{-1+\delta}^{\delta}\gR_{sb,1,\theta^*}^2 (t) \dd{t} \right] \nonumber\\ 
    + \int_{-1+\delta}^{\delta} \int_{-1}^1 \gR_{int,\theta^*}^2(x,t) \dd{x} \dd{t} + 2 C_{1b} \left[ \left(\int_{-1+\delta}^{\delta} \gR^2_{sb,-1,\theta^*} \dd{t}\right)^\frac{1}{2} + \left(\int_{-1+\delta}^{\delta} \gR^2_{sb,1,\theta^*} \dd{t}\right)^\frac{1}{2} \right]
    \label{eq:int_burger_proof_6}
\end{align}

% \note{break this align environment into two pieces to prevent the white space that is appearing before it} \allowdisplaybreaks[1] does the job
Using integral form of Gr$\ddot{o}$nwall's inequality on (\ref{eq:int_burger_proof_6})
\begin{align}
     \int_{-1}^1 \hat u^2(x, \bar\delta) dx &\leq C_T + \int_{-1 + \delta}^{\bar\delta} C_T C e^{\int_t^{\delta} C ds} dt 
     \leq \left[ 1 + \int_{-1 + \delta}^{\bar\delta} Ce^{C} dt \right] C_T
     \label{eq:int_burger_proof_7}
\end{align}
Integrating (\ref{eq:int_burger_proof_7}) over $\bar\delta$ together with the definition of generalization error (\ref{eq:gen_error}) we get
\begin{align}
    \int_{-1 + \delta}^{\delta} \int_{-1}^1 \hat u^2(x, \bar \delta) dx d\bar\delta &\leq C_T \int_{-1 + \delta}^{\delta}\left[ 1 + \int_{-1 + \delta}^{\bar\delta} Ce^{C} dt \right] d\bar\delta \nonumber\\
    \gE_G^2 &\leq \left[ 1 + Ce^{C} \right] C_T
    \label{eq:int_burger_proof_8}
\end{align}
\end{proof}

\subsection{Making The Data Dependence Explicit in the Bounds for $1+1$ Burgers' PDE} \label{appendix:subsec:burger_gen_error_bound}

%\pa{Quadrature Rule}
Suppose $g\ \colon\ \mathbb{D} \rightarrow \R^m$ such that $g \in Z^* \subset L^p(\mathbb{D},\R^m)$ and $\mathbb{D} \subset \R^{\bar d}$ and suppose that the integral that needs to be approximated is,
\begin{align}
    \bar g \coloneqq \int_\mathbb{D} g(y)\ dy
    \label{eq:integral}
\end{align}
where $dy$ denotes the $\bar d$-dimensional Lebesgue measure. To approximate this integral by the quadrature rule, we need \begin{inparaenum}[(i)] \item the quadrature points $y_i \in \mathbb{D}$ for $1\leq i \leq N$ for some $N \in \N$ and \item weights $w_i$ with $w_i \in \R_+$ \end{inparaenum}. Then we can approximate \plaineqref{eq:integral} by the quadrature,
\begin{align}
    \bar g_N\ \coloneqq\ \sum_{i=1}^N w_ig(y_i)
    \label{eq:quad_rule}
\end{align}
We recall that for implementing the Gaussian quadrature rule the quadrature points are roots of appropriate Legendre polynomials and the quadrature weights are the integrals of the associated Lagrange interpolation polynomials. \cite{devore1984quad} Then its known \cite{siddhartha2022generror} that, there exists a function $C_{quad}$ s.t the error of this approximation is bounded as,
\begin{align}
    |\ \bar g - \bar g_N\ | \leq C_{quad}(\norm{g}_{Z^*}, \bar d)\ N^{-\alpha}\ ,\ \mbox{for some}\ \alpha > 0
    \label{eq:quad_error}
\end{align}

\subsubsection{Applying Quadrature Rule on Theorem \ref{th:int_burger_gen_error_bound} }
% For a choice of weights and quadrature points,
In light of the  above we choose a training set consisting of quadrature points of the following $4$ types,
\begin{itemize}
    \item $N_{int}$ number of interior points $\{x_n, t_{n,\delta}\}$  where $x_n \in [-1,1]$ and $t_{n,\delta} \in [-1+\delta, \delta]$
    \item $N_{sb}$ number of spatial boundary points $\{-1,t_{n,\delta}\}$  where $t_{n,\delta} \in [-1+\delta, \delta]$
    \item $N_{sb}$ number of spatial boundary points $\{1,t_{n,\delta}\}$  where $t_{n,\delta} \in [-1+\delta, \delta]$
    \item $N_{tb}$ number of points at initial conditions, $\{x_n,-1+\delta\}$ where $x_n \in [-1,1]$
\end{itemize}

For appropriate weights (one for every of the sample points chosen above) and and using the above sample points as the quadrature points, the loss function corresponding to the residuals in Equation (\ref{eq:burger_pdes_residual:1}), (\ref{eq:burger_pdes_residual:2}) and (\ref{eq:burger_pdes_residual:3}), evaluated for a predictor with weights $\theta^*$, can then be written as,
\begin{align} \label{eq:sid_training_error}
\small
    \gL(\theta^*)\ =\ \gE_T^2\ \coloneqq \frac{1}{N_{tb}}\underbrace{\sum_{n=1}^{N_{tb}} w_n^{tb}|\mathcal{R}_{tb,\theta^*}(x_n)|^2}_{(\gE_T^{tb})^2} 
    + \frac{1}{N_{sb}}\underbrace{\sum_{n=1}^{N_{sb}} w_n^{sb}|\mathcal{R}_{sb,-1,\theta^*}(t_{n,\delta})|^2}_{(\gE_T^{sb,-1})^2} 
    \nonumber\\
    + \frac{1}{N_{sb}}\underbrace{\sum_{n=1}^{N_{sb}} w_n^{sb}|\mathcal{R}_{sb,1,\theta^*}(t_{n,\delta})|^2}_{(\gE_T^{sb,1})^2} 
    + \frac{\lambda}{N_{int}}\underbrace{ \sum_{n=1}^{N_{int}} w_n^{int}|\mathcal{R}_{int,\theta^*}(x_n, t_{n,\delta})|^2}_{(\gE_T^{int})^2}
\end{align}

% \note{Explain the $w$s in the above - before you state this loss function!}

In contrast to the above note that the experiments shown in Section~\ref{sec:exp.burgers} for its empirical loss used randomly chosen collocation points and choose  all $w$-s and $\lambda$ to be equal to $1$. But, via using the quadrature rule, for the above loss, one can get the following risk bound where the dependency on the number of collocation points becomes explicit, 

% \note{You need to define what is a ``constant of quadrature'' before the following theorem can be stated - otherwise the following theorem is not making sense. Nobody knows what these $C_{quad}$ things are!}

\clearpage 

\begin{theorem}\label{th:burger_gen_error_bound}
Let $u \in C^1((-1+\delta,\delta) \times (-1,1))$ be the unique solution of the viscous scalar conservation law for any $k \geq 1$. Let $u^* = u_{\theta^*}$ be any neural surrogate solution. Then its generalization error (\ref{eq:gen_error}) is bounded by
    \begin{align}
        \gE_G^2(u^*) &\leq \left(1 + Ce^{C}\right) \left[ \sum_{n=1}^{N_{tb}} w_n^{tb}|\mathcal{R}_{tb,\theta^*}(x_n)|^2 + \sum_{n=1}^{N_{int}} w_n^{int}|\mathcal{R}_{int,\theta^*}(x_n, t_{n,\delta})|^2 \right. \nonumber\\
        &+ \left. 2 C_{2b} \left(\sum_{n=1}^{N_{sb}} w_n^{sb}|\mathcal{R}_{sb,-1,\theta^*}(t_{n,\delta})|^2 + \sum_{n=1}^{N_{sb}} w_n^{sb}|\mathcal{R}_{sb,1,\theta^*}(t_{n,\delta})|^2\right) + 2 C_{1b} \left(\gE_T^{sb,-1} + \gE_T^{sb,1}\right) \right. \nonumber\\
        &+ \left. \frac{C_{quad}^{tb}}{N_{tb}^{\alpha_{tb}}} + \frac{C_{quad}^{int}}{N_{int}^{\alpha_{int}}} + 2C_{2b} \left( \frac{\left( C_{quad}^{sb,-1} + C_{quad}^{sb,1} \right)}{N_{sb}^{\alpha_{sb}}} \right) + 2 C_{1b} \left( \frac{\left( C_{quad}^{sb,-1} + C_{quad}^{sb,1} \right)} {N_{sb}^{\frac{\alpha_{sb}}{2}}} \right)  \right]
        \label{eq:burger_bound__th_1}
    \end{align}
    where $C = 1 + 2C_{u_x}$, with
    \begin{align}
        C_{u_x} &= \norm{u_x}_{L^\infty} = \norm{\frac{1}{t-1}}_{L^\infty([-1+\delta,\delta])} \nonumber \\
        C_{1b} &= \norm{u(1,t)}^2_{L^\infty} = \norm{\frac{1}{1-t}}^2_{L^\infty([-1+\delta,\delta])} \nonumber \\
        C_{2b} &= \norm{u_{\theta^*}(1,t)}_{L^\infty([-1+\delta,\delta])} + \frac{3}{2}\norm{\frac{1}{t-1}}_{L^\infty([-1+\delta,\delta])}
        \label{eq:burger_bound__th_2}
    \end{align}
    and $ C_{quad}^{tb} = C_{quad}^{tb}\left( \norm{\gR_{tb,\theta^*}^2}_{C_k} \right)$, $C_{quad}^{int} = C_{quad}^{int}\left( \norm{\gR_{int,\theta^*}^2}_C^{k-2} \right)$, $C_{quad}^{sb,-1} = C_{quad}^{sb,-1} \left( \norm{\gR_{sb,-1,\theta^*}}_{C^k} \right)$, $C_{quad}^{sb,1} = C_{quad}^{sb,1}\left( \norm{\gR_{sb,1,\theta^*}^2}_{C^k} \right)$ are constants of the quadrature bound.
\end{theorem}

%The proof for this theorem is shown in Appendix \ref{proof:burger_gen_error_bound}.

\begin{proof}
    In equation (\ref{eq:int_burger_proof_8}) within the proof of Theorem \ref{th:int_burger_gen_error_bound} we see that
    \begin{align}
        \int_{-1 + \delta}^{\delta} \int_{-1}^1 \hat u^2(x, \bar \delta) dx d\bar\delta &\leq C_T \int_{-1 + \delta}^{\delta}\left[ 1 + \int_{-1 + \delta}^{\bar\delta} Ce^{C} dt \right] d\bar\delta \nonumber\\
        \gE_G^2 &\leq \left[ 1 + Ce^{C} \right] C_T \nonumber
    \end{align}
    where
    \begin{align}
        C_T = \underbrace{\int_{-1}^1 \gR_{tb,\theta^*}(x) dx}_1 + \underbrace{2 C_{2b} \left[\int_{-1+\delta}^{\delta}\gR_{sb,-1,\theta^*}^2 (t) dt + \int_{-1+\delta}^{\delta}\gR_{sb,1,\theta^*}^2 (t) dt \right]}_3 \nonumber\\ 
        + \underbrace{\int_{-1+\delta}^{\delta} \int_{-1}^1 \gR_{int,\theta^*}^2(x,t) dx dt}_2 + \underbrace{2 C_{1b} \left[ \left(\int_{-1+\delta}^{\delta} \gR^2_{sb,-1,\theta^*} dt\right)^\frac{1}{2} + \left(\int_{-1+\delta}^{\delta} \gR^2_{sb,1,\theta^*} dt\right)^\frac{1}{2} \right]}_4
        \label{eq:quad_burger_proof_1}
    \end{align}

    % \note{something is wrong with the alignment of the equations!}
    
    Applying quadrature bounds on \plaineqref{eq:quad_burger_proof_1} we get,
    \begin{align}
        C_T &\leq \underbrace{\sum_{n=1}^{N_{tb}} w_n^{tb}|\gR_{tb,\theta^*}(x_n)|^2 + C_{quad}^{tb}(\norm{\gR_{tb,\theta^*}}_{C^k})N_{tb}^{-\alpha_{tb}}}_{1} 
        \nonumber\\
        &+ \underbrace{\sum_{n=1}^{N_{int}} w_n^{int}|\gR_{int,\theta^*}(x_n,t_{n,\delta})|^2 + C_{quad}^{int}(\norm{\gR_{int,\theta^*}}_{C^{k-2}})N_{int}^{-\alpha_{int}}}_{2} \nonumber\\
        &+ \underbrace{2C_{2b} \left[ \sum_{n=1}^{N_{sb}} w_n^{sb} |\gR_{sb,-1,\theta^*}(t_{n,\delta})|^2 + \sum_{n=1}^{N_{sb}} w_n^{sb} |\gR_{sb,1,\theta^*}(t_{n,\delta})|^2 \right.}_{3} 
        \nonumber\\
        &+ \underbrace{\left.\left( C_{quad}^{sb}(\norm{\gR_{sb,-1,\theta^*}}_{C^k}) + C_{quad}^{sb}(\norm{\gR_{sb,1,\theta^*}}_{C^k}) \right) N_{sb}^{\alpha_{sb}} \right]}_{3} \nonumber\\
        &+ \underbrace{2C_{1b} \left[ \left(\sum_{n=1}^{N_{sb}} w_n^{sb} |\gR_{sb,-1,\theta^*}(t_{n,\delta})|^2\right)^\frac{1}{2} + \left(\sum_{n=1}^{N_{sb}} w_n^{sb} |\gR_{sb,1,\theta^*}(t_{n,\delta})|^2\right)^\frac{1}{2} \right.}_{4}
        \nonumber\\
        &+ \underbrace{ \left.\left( C_{quad}^{sb}(\norm{\gR_{sb,-1,\theta^*}}_{C^k}) + C_{quad}^{sb}(\norm{\gR_{sb,1,\theta^*}}_{C^k}) \right)^\frac{1}{2} N_{sb}^\frac{\alpha_{sb}}{2} \right]}_{4}
    \end{align}
    Replacing the sums of residuals by training error we get
    \begin{align}
        \gE_G^2 &\leq \left(1 + Ce^{C}\right) \left[ \left( \gE_T^{tb} \right)^2 + \left( \gE_T^{int} \right)^2 + 2 C_{2b} \left(\left(\gE_T^{sb,-1}\right)^2 + \left(\gE_T^{sb,1}\right)^2\right) + 2 C_{1b} \left(\gE_T^{sb,-1} + \gE_T^{sb,1}\right) \right] \nonumber\\
        &+ \left(1 + Ce^{C}\right) \left[ C_{quad}^{tb} N_{tb}^{-\alpha_{tb}} + C_{quad}^{int} N_{int}^{-\alpha_{int}} + 2 C_{2b} \left( \left( C_{quad}^{sb,-1} + C_{quad}^{sb,1} \right) N_{sb}^{-\alpha_{sb}} \right) \right.
        \nonumber\\
        &+ 2 C_{1b} \left. \left( \left( C_{quad}^{sb,-1} + C_{quad}^{sb,1} \right) N_{sb}^{\frac{-\alpha_{sb}}{2}} \right)  \right]
    \end{align}
\end{proof}

%\section{Plots}
% \subsection{Evaluating $C_2$}
% \begin{align*}
%     C_2 &= \underbrace{\int_D \norm{\gR_t}_2^2 \dd{\vx}}_{I_1} + \underbrace{\int_{\Tilde{\Omega}} \norm{\gR_{pde}}_2^2  \dd{\vx} \dd{t}}_{I_2} + 2^2 \norm{\nabla \vu_\theta}_{L^\infty(\Omega)} \underbrace{\int_{\Tilde{\Omega}} \norm{\vu_\theta}_2^2 \dd{\vx} \dd{t}}_{I_3} \\
%     &+ 2^2 \max_{t=\frac{1}{\sqrt{2}} + \delta, \delta}\left\{\abs{\frac{1-2t}{1-2t^2}} + \abs{\frac{1}{1-2t^2}}, \abs{\frac{1}{1-2t^2}} + \abs{\frac{1+2t}{1-2t^2}} \right\} \underbrace{\int_{\Tilde{\Omega}} \norm{\vu}_2^2 \dd{\vx} \dd{t}}_{I_4}  \\ %1
% \end{align*}

\clearpage 
\section{Plotting the Behaviour of RHS of Equation \ref{eq:int_burger_bound_th_1} for Varying Widths}
\begin{figure}[!h]
  \centering
    \includegraphics[width=0.49\textwidth]{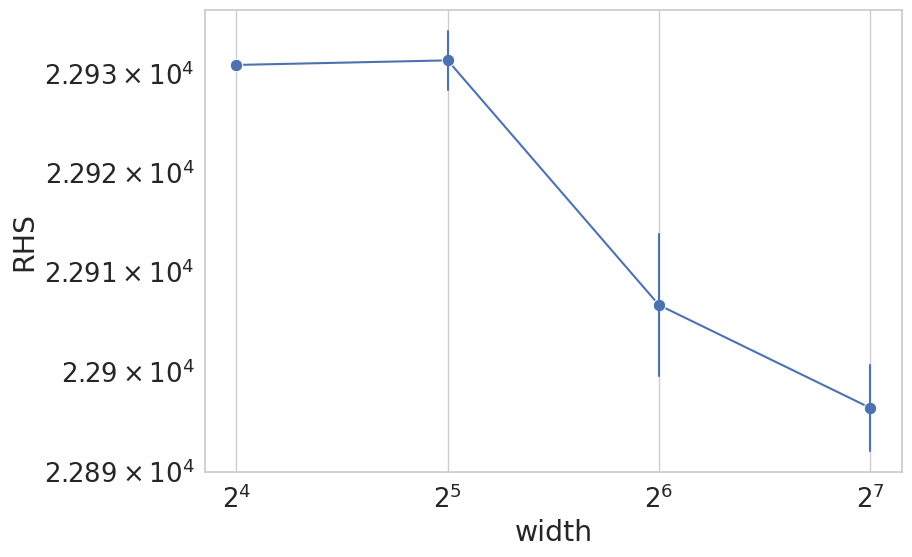}
  \caption{The above plot tracks the RHS of equation (\ref{eq:int_burger_bound_th_1}) in Theorem \ref{th:int_burger_gen_error_bound} for training a depth $2$ net at  different widths towards solving equation \ref{eq:burger_pdes:6} at $\delta = \frac{1}{2}$}
  \label{fig:plot_Burgers_varywidth}
\end{figure}

\section{A Study of the Approximate Invariance of the Training Time for the 1D Burgers' PDE vs Proximity to Singularity}

\begin{figure}[htbp!]
    \centering
    \begin{subfigure}[h]{0.49\textwidth}
        \includegraphics[width=\textwidth]{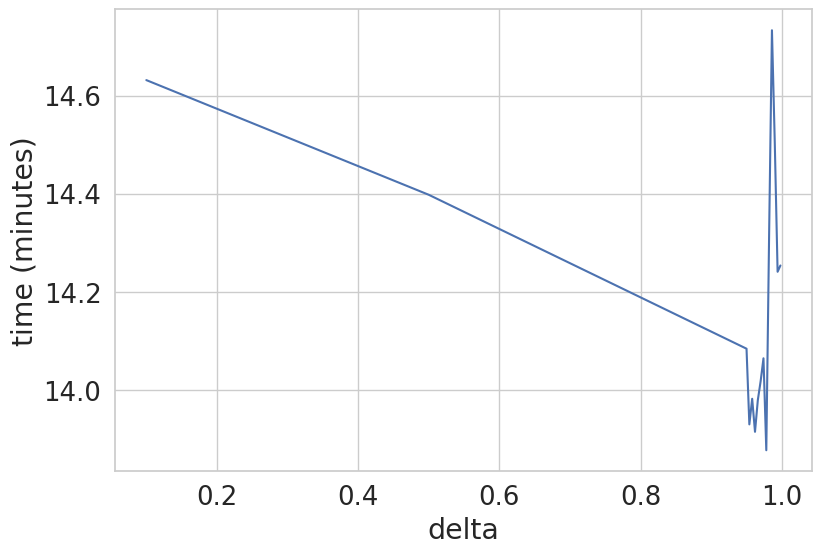}
        \caption{width=30}
        \label{fig:plot_}
    \end{subfigure}
    \hfill
    \begin{subfigure}[h]{0.49\textwidth}
        \includegraphics[width=\textwidth]{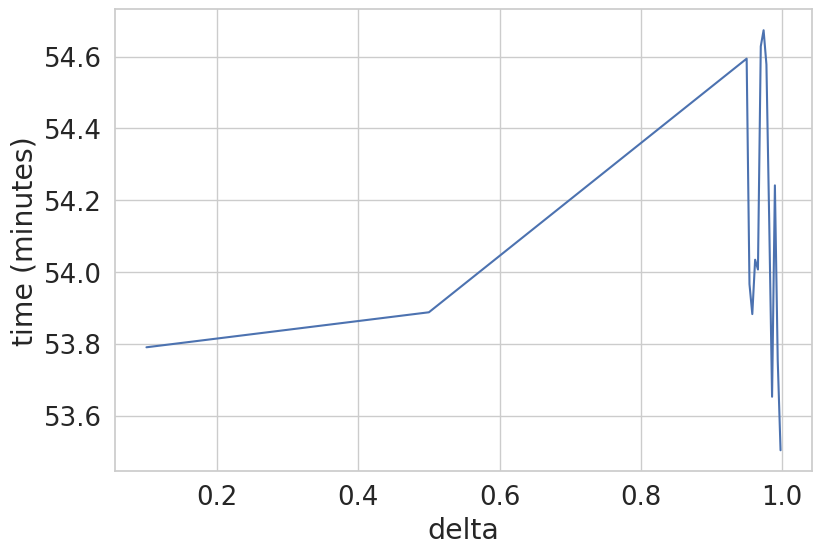}
        \caption{width=300}
        \label{fig:plot_}
    \end{subfigure}
    \caption{The above plots show that the time taken to train a PINN on equation \ref{eq:burger_pdes:6} barely changes for different values of $\delta$ - a measure of proximity to blow-up and that this holds at two widely separated widths of the net.}
    \vspace{-0.30em}
    \label{fig:plot_time_burgers}
\end{figure}

\end{document}